\documentclass{article}

\usepackage{microtype}
\usepackage{booktabs} 

\usepackage[accepted]{icml2022}

\usepackage{amsmath}
\usepackage{bm}
\usepackage{bbm}
\usepackage{amssymb}
\usepackage{dsfont}
\usepackage{mathtools}

\usepackage{amsthm}
\usepackage{multirow}

\newcommand{\specialcell}[2][c]{%
  \begin{tabular}[#1]{@{}c@{}}#2\end{tabular}}
\usepackage{ragged2e}
\usepackage{enumitem}
\usepackage{siunitx}
 \usepackage{xcolor}
\usepackage{graphicx}
\usepackage[justification=justified]{caption}
\usepackage{subcaption}
\usepackage[capitalize,noabbrev]{cleveref}
\usepackage{multicol}

\icmltitlerunning{Optimal Clipping and Magnitude-aware Differentiation for Improved Quantization-aware Training}

\theoremstyle{plain}
\newtheorem{theorem}{Theorem}[section]
\newtheorem{proposition}[theorem]{Proposition}

\newtheorem{corollary}[theorem]{Corollary}
\theoremstyle{definition}

\theoremstyle{remark}

\usepackage{setspace}
\usepackage{titlesec}
\usepackage{caption} 
\begin{document}

\twocolumn[
\icmltitle{Optimal Clipping and Magnitude-aware Differentiation \\ for Improved Quantization-aware Training}



\begin{icmlauthorlist}
\icmlauthor{Charbel Sakr}{to}
\icmlauthor{Steve Dai}{to}
\icmlauthor{Rangharajan Venkatesan}{to}
\icmlauthor{Brian Zimmer}{to}
\icmlauthor{William J. Dally}{to}
\icmlauthor{Brucek Khailany}{to}
\end{icmlauthorlist}

\icmlaffiliation{to}{The authors are with NVIDIA Corporation, Santa Clara, CA 95051 USA}
\icmlcorrespondingauthor{Charbel Sakr}{csakr@nvidia.com}

\icmlkeywords{Machine Learning, ICML}

\vskip 0.3in
]



\printAffiliationsAndNotice{}  

\begin{abstract}
Data clipping is crucial in reducing noise in quantization operations and improving the achievable accuracy of quantization-aware training (QAT). Current practices rely on heuristics to set clipping threshold scalars and cannot be shown to be optimal. We propose Optimally Clipped Tensors And Vectors (OCTAV), a recursive algorithm to determine MSE-optimal clipping scalars. Derived from the fast Newton-Raphson method, OCTAV finds optimal clipping scalars on the fly, for every tensor, at every iteration of the QAT routine. Thus, the QAT algorithm is formulated with provably minimum quantization noise at each step. In addition, we reveal limitations in common gradient estimation techniques in QAT and propose magnitude-aware differentiation as a remedy to further improve accuracy. Experimentally, OCTAV-enabled QAT achieves state-of-the-art accuracy on multiple tasks. These include training-from-scratch and retraining ResNets and MobileNets on ImageNet, and Squad fine-tuning using BERT models, where OCTAV-enabled QAT consistently preserves accuracy at low precision (4-to-6-bits). Our results require no modifications to the baseline training recipe, except for the insertion of quantization operations where appropriate.
\end{abstract}

\section{Introduction}
\label{sec:intro}
Deep neural networks (DNNs) are powerful models achieving state-of-the-art accuracy on various cognitive tasks such as image classification, object detection, and natural language processing \cite{lecunNature}. However, DNN successes have been achieved at the expense of a high computational and parameter complexity. Indeed, networks commonly require around 4 billion multiply-accumulates (MACs) \cite{resnet} or have over 100 million parameters \cite{deepface}. With further progress in deep learning, the growth of DNN resource requirements shows no sign of slowing down \cite{bianco2018benchmark}. Fortunately, reduced-precision implementation has been shown to largely lower the computational complexity of deep learning models \cite{ibmIcml15}.

Reduced-precision deep learning is widely adopted and multi-faceted. One option is to perform post-training quantization (PTQ), which benefits inference only and consists of taking a pretrained network and implementing it in reduced-precision, with no retraining allowed. While conceptually simple, PTQ is challenging, as determining an accuracy-preserving quantization strategy is non-trivial. In recent works, proposed practical PTQ solutions include speculative hybrid high/low precision number formats \cite{biscaledDAC,predictivenet} and Batch-Norm-guided data free quantization \cite{nagelDFQ}.

The results presented in our work apply to any quantization setup, including PTQ. However, the main feature is the ability to optimize quantization metadata on the fly. This is most arguably effective for another facet: quantization-aware training (QAT), where weights and activations are quantized during training. A superset of this problem is fully quantized training (FQT) \cite{sakrICLRfx}, where gradients and weight updates are also quantized. While we focus on QAT here, an interesting and important extension of our work is to apply our results to FQT.

\subsection{Quantization-aware training and related works} 
There are three main use cases for QAT:
\textbf{training-from-scratch} where the starting point is a randomly initialized network; \textbf{retraining} where a pretrained model is quantized and retrained for a short time on the same dataset; and \textbf{fine-tuning} where the starting point is a model pretrained on one dataset and  trained on another.

Early works on QAT showed that binary-weighted \cite{binaryConnect} and fully-binarized \cite{binaryNet} networks can be accurately trained on simple models and datasets. To improve accuracy in more difficult tasks, DoReFa-Net \cite{dorefa} increased the forward precision to 4-bit and used \textit{max-scaling}, i.e., matching the largest quantized representation to the largest value in the set of elements (tensor or vector) to be quantized.

An advantage of max-scaling is that it is well-defined. Indeed, max-scaled QAT can be implemented using the same training recipe as a full precision baseline, simply by inserting the quantization operations where appropriate. Thus, there are no hyperparameters required, and results can be readily reproduced. Unfortunately, max-scaling incurs large amounts of quantization noise, harming accuracy. Nevertheless, quantization fidelity can be improved with clipping, which has been researched by several recent works.

Making the \textit{clipping scalar} a learned parameter, as was done in PACT \cite{pact}, has been shown to significantly improve accuracy.  However, by virtue of the added learnable parameters, a QAT-dedicated training recipe is required. As such, PACT is highly sensitive to hyperparameter tuning and is therefore difficult to reproduce and cannot easily generalize. An advantage of PACT is that, much like max-scaling, it computes quantization metadata on the fly. We term such a scheme \textbf{dynamic quantization}, which is useful for setups with time-varying tensor statistics, such as training-from-scratch and fine-tuning.

Alternatively, \textbf{static quantization} was explored by \cite{wu2020integer} who used calibration on pretrained data in order to fix the clipping scalars at the start of QAT. Such an approach can only be applied when tensor statistics minimally change over time, e.g., in a short retraining or fine-tuning setup. Further, the choice of calibration strategy is closely tied to the performance of QAT, necessitating network-specific exploration. Percentile calibration was shown to be robust \cite{wu2020integer}, and state-of-the-art retraining accuracy was obtained through extensive calibration exploration \cite{paretoCVPR21}. An advantage of static quantization is its conceptual simplicity; once a good calibration strategy is identified, a QAT routine can be readily and reproducibly implemented.

The above works all use uniform quantization, also called integer quantization \cite{wu2020integer} and fixed-point quantization \cite{sakrICLRfx}. In this work, we also focus on uniform quantization, which is well-suited for efficient hardware implementations, particularly at low (less-than-8-bit) precision \cite{hanEIE}.

Still, we do note that recent works have shown promising results for QAT under non-uniform quantization: structured quantization, such as that derived from low precision floating-point \cite{ibmNips18,IBMNIPS2019} and logarithmic \cite{lee2017lognet, nvidiaLNS} number systems, as well as custom formats, such as Flexpoint \cite{flexpoint}, AdaptivFloat \cite{tambeAdaptivfloat}, and LQNets \cite{lqnet}. We derive all our theoretical and experimental results for uniform quantization; however, our work can be extended to non-uniform quantization.

Finally, in our work, we use the exact same training recipe as the full-precision baseline for training-from-scratch and fine-tuning. For retraining, we use a shortened version of the training-from-scratch recipe. Recent works have attempted to improve accuracy using QAT-specific training techniques, such as distillation \cite{distillationCVPR2020}. Such works are orthogonal to ours; our methods can be inserted on top of any training routine.

\subsection{Contributions}
None of the prior arts provides guarantees on the optimality of the chosen clipping scalars. Most works use a strategy believed to be adequate, e.g., relying on the training algorithm \cite{pact}, or yielding small quantization noise on calibrated data \cite{wu2020integer}. In contrast, our work makes the following contributions:
\begin{itemize}[leftmargin=*,itemsep=0.1\baselineskip,topsep=0pt]
    \item We derive OCTAV: a fast recursive algorithm based on the Newton-Raphson method to determine MSE-minimizing clipping scalars. With OCTAV, optimal quantization metadata can be computed for every tensor, at every iteration of the QAT routine. Thus, the QAT algorithm is formulated with minimum quantization noise at each iteration.
    \item We analyze common candidates for quantized gradient estimation, for which we reveal risks of gradient explosion and partial premature stoppage of convergence. We avoid these risks by proposing magnitude-aware differentiation, which leads to a noticeable improvement in QAT accuracy.
    \item We show that OCTAV-enabled training-from-scratch QAT achieves state-of-art accuracy on several ImageNet benchmarks. Indeed, 4-bit training of ResNet-50, ResNet-18, ResNet-101, and MobileNet-V2 models results in less-than-1\% accuracy degradation compared to the full precision baseline. We also provide promising results on the much-harder-to-quantize MobileNet-V3-Small and MobileNet-V3-Large. For all results, no modification to the baseline training recipe is made.
    \item We find OCTAV-enabled QAT to always yield highly accurate solutions in 4-bit retraining. We find static quantization more accurate for large models, such as ResNets; and thus propose static-OCTAV for fast calibration yielding high accuracy. In contrast, small models such as MobileNets require dynamic quantization for retraining, and OCTAV is shown to be far superior to any other strategy.
    \item Finally, we find OCTAV-enabled QAT to be most appropriate for Squad fine-tuning of BERT models. Even when restricted to static quantization, we find that static-OCTAV consistently outperforms other calibration methods.
\end{itemize}
\section{Clipped Quantization}
\label{sec:clipped_quantization}

Consider some data $x$ derived from a distribution $f_X()$. We define $B$-bit quantization as the process of mapping $x$ to one of $2^B$ predefined levels $\{r_i\}_{i=1}^{2^B}$. The quantized data is obtained as:
$    \mathbbm{Q}(x) = \arg\min_{\{r_i\}_{i=1}^{2^B}}|x-r_i|.$
The choice of $\{r_i\}_{i=1}^{2^B}$ is crucial in setting the fidelity of quantization, which we metricize via the mean squared error (MSE):
\begin{align}
    \label{eqn:mse_definition}
    J = \mathbbm{E}\left[ \left(\mathbbm{Q}(X) - X \right)^2 \right]
\end{align}
For unconstrained quantization, this metric can be minimized using the Lloyd-Max algorithm \cite{lloyd_max}.

Among uniformly constrained quantizers, we first introduce the maxed-scaled one. Assume there exists a scalar $s_{\max}$ such that $f_X(x)=0$ for $|x|>s_{\max}$. In practice, $s_{\max}$ can be the largest available element in absolute value. The max-scaled quantizer assigns the levels $\{r_i\}_{i=1}^{2^B}$ as an arithmetic progression on $\left[-s_{\max},s_{\max}\right]$\footnote{We assume signed data, without loss of generality. In Appendix \ref{appendix:unsigned}, we list all required modifications for our results to apply to the unsigned case, i.e., quantization over $\left[0,s_{\max}\right]$.}. Thus, the max-scaled quantization operation is given by:
\begin{align}
    \label{eqn:max_scaled_quantization}
    \mathbbm{Q}(x) = s_{\max}\cdot2^{1-B}\cdot\text{round}\left({x\cdot2^{B-1}}/{s_{\max}} \right)
\end{align}
with the rounding operation being applied on integers\footnote{For notational simplicity and mathematical tractability, boundary effects introduced by number systems (e.g., two's complement) are neglected. They are, however, implemented in our experiments.}.
This quantizer has been extensively studied in signal processing \cite{goelQnoise} and machine learning \cite{sakrICML} and its MSE, derived using an additive model of quantization noise, is given by $J=s_{\max}^2\frac{4^{-B}}3$. 

Often, max-scaling is data inefficient due to its large quantization range and thus step size. A simple, but powerful method to improve uniform quantization is to allow for data clipping \cite{sakrTSP,gonugondlaICCAD}. Specifically, a narrower quantization interval $[-s,s]$ is used, with the clipping scalar $s<s_{\max}$, and the quantization operation given by:
\begin{align}
    &\mathbbm{Q}(x) =  \text{clip}\left(s\cdot2^{1-B}\cdot\text{round}\left({x\cdot2^{B-1}}/{s}\right) ,-s,s\right) \nonumber \\
    &=\begin{cases}
    -s \quad \text{if } x<-s\\
    s\cdot2^{1-B}\cdot\text{round}\left({x\cdot2^{B-1}}/{s}\right) ~ \text{if } x\in[-s,s]\\
    s \quad \text{if } x>s
    \end{cases} 
    \label{eqn:clipped_quantization_definition}
\end{align}
    
With clipping, the MSE in \eqref{eqn:mse_definition} depends on $s$ and is given by:
\begin{align}
    \label{eqn:clipped_quantization_mse}
    J(s) = \frac{4^{-B}}3s^2\int_0^sf_{|X|}(x)dx + \int_s^{\infty}(s-x)^2f_{|X|}(x)dx
\end{align}
where $f_{|X|}()$ is the distribution of the absolute value of the data. Equation \eqref{eqn:clipped_quantization_mse} is obtained by evaluating \eqref{eqn:mse_definition} using the law of total expectation; the aforementioned additive noise model is assumed on the discretization interval $[-s,s]$ and the definition of MSE is used when clipping occurs. {For discretization noise, the term $\frac{s^24^{-B}}{3}$ does not require \emph{a priori} knowledge of data distribution. It is obtained through sampling theory where quantization noise arises via approximating the neighborhood of a quantization level of \emph{any} distribution as a local rectangle \cite{widrowBook}. }

\begin{figure}[!t]
\begin{center}
    \begin{subfigure}[t]{0.49\textwidth}
    \centering
        \includegraphics[trim=0 2cm 0 2cm, clip, width = 0.99\linewidth,page=1]{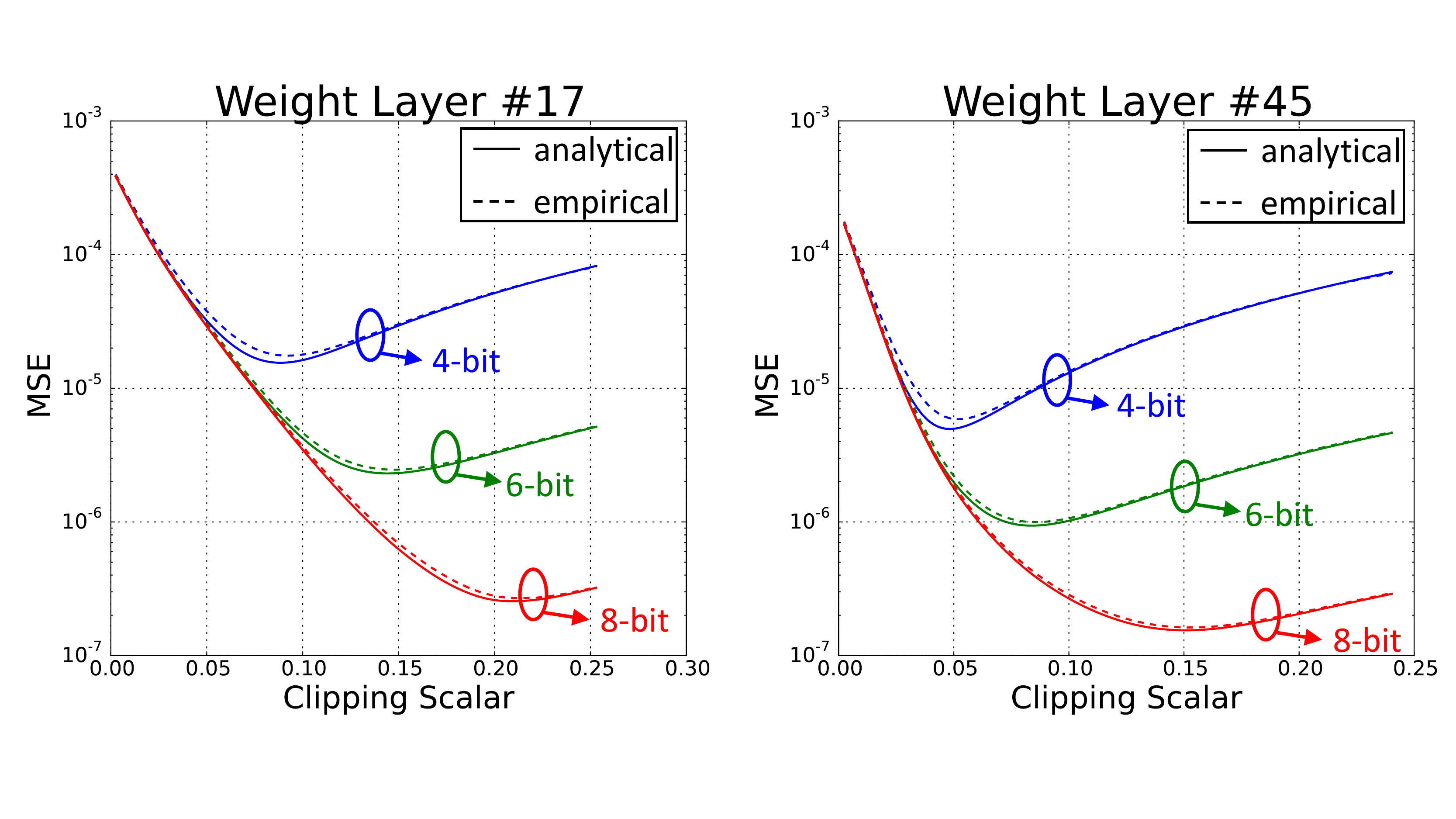}
    \caption{}
    \end{subfigure}
    \begin{subfigure}[t]{0.49\textwidth}
    \centering
        \includegraphics[trim=0 2cm 0 1cm, clip, width = 0.99\linewidth,page=2]{figures/paper_figures.pdf}
    \caption{}
    \end{subfigure}
\end{center}
\caption{Sweep of the quantization MSE as a function of the clipping scalar $s$ for two arbitrary weight (a) and activation (b) layers in a pretrained ResNet-50 model. Activation data is obtained by sampling a random input batch from the training set. Solid lines are obtained by evaluating the MSE formula in \eqref{eqn:clipped_quantization_mse} using histograms and numerical integration. Dashed lines are obtained by empirically evaluating \eqref{eqn:mse_definition}, i.e., quantizing each tensor element according to \eqref{eqn:clipped_quantization_definition} and averaging the resulting squared errors.}
\label{fig:mse_vs_scalars_curves}
\end{figure}

Figure \ref{fig:mse_vs_scalars_curves} shows variations in quantization MSE as a function of clipping scalar. We use data from a trained ResNet-50 model \cite{resnet}, with activations corresponding to a random input batch sampled from the ImageNet dataset \cite{imagenet}. We depict both the evaluation of \eqref{eqn:clipped_quantization_mse} using histograms and numerical integration, and the empirically measured MSE in \eqref{eqn:mse_definition} via element-wise quantization and squared error averaging. We observe the following:
\begin{itemize}[leftmargin=*,itemsep=0.1\baselineskip,topsep=0pt]
    \item Our formula closely matches the empirical MSE. Thus, we use it as a building block for our upcoming analyses.
    \item There exists an optimal scalar $s^*$ minimizing the MSE. This optimum balances the trade-off between discretization and clipping noise. When $s<s^*$, excess clipping leads to an increase in $J$ in spite of smaller discretization noise. Conversely, when $s>s^*$, clipping is minimal but the larger quantization step size causes an increase in discretization noise and $J$.
    \item The optimal scalar $s^*$ is a function of both data distribution $f_{X}()$ and number of bits $B$. The dependence on $f_X()$ is identified by virtue of $s^*$ being different for different layers (e.g., when $B=4$, $s^*$ is approximately $0.1$ and $0.05$ for weight layers \#17 and \#45, respectively). The dependence on $B$ is identified by virtue of $s^*$ varying with precision when data is unchanged (e.g., for activation layer \#13, $s^* $ is approximately $1.0$ and $2.0$ for $B=4$ and $B=8$, respectively).
\end{itemize}

Finding $s^*$ can be done \textbf{offline} through brute force search, i.e., sweeping the value of $s$. However, this task is highly time-consuming and hard to implement dynamically. The analytical evaluation of \eqref{eqn:clipped_quantization_mse} requires histograms to estimate $f_X()$ and numerical integration. Similarly, an empirical evaluation requires successive rounding and reduction operations on large tensors.
In the next section, we present a method to determine the optimal clipping scalar $s^*$ \textbf{online}.

\section{Optimally Clipped Tensors And Vectors}
\label{sec:octav}
We present our main theoretical result: a recursive formula to analytically determine $s^*$.
\begin{theorem}
\label{thm:octav}
Given a data distribution $f_X()$, the clipping scalar $s^*$ minimizing the clipped quantization MSE in \eqref{eqn:clipped_quantization_mse} can be found by assigning a random guess $s_1$ and recursively computing $\{s_n\}_{n>1}$ until convergence using:
\begin{align}
\label{eqn:octav_theorem}
    s_{n+1} = \frac{\mathbbm{E}\left[|X|\cdot \mathbbm{1}_{\{|X|>s_n\}} \right]}{\frac{4^{-B}}{3}\mathbbm{E}\left[\mathbbm{1}_{\{|X|\leq s_n\}}\right] +\mathbbm{E}\left[ \mathbbm{1}_{\{|X|>s_n\}}\right]}
\end{align}
\end{theorem}
\begin{proof}
We provide the complete proof in Appendix \ref{appendix:octav_proof}. For the benefit of the interested reader, we  here mention the main idea behind the result. It consists of using the Newton-Raphson algorithm, which recursively computes 
$s_{n+1} = s_n - {J'(s_n)}/{J''(s_n)}$.
In Appendix \ref{appendix:octav_proof}, we show how first and second derivatives of $J(s)$ are derived to obtain \eqref{eqn:octav_theorem}.
\end{proof}
Theorem \ref{thm:octav} applies to an arbitrary distribution, and the following corollary applies to tensor and vector quantization.
\begin{corollary}
\label{corr:octav}
The clipping scalar $s^*$ minimizing the clipped quantization MSE in a tensor or vector $\vec{t}$ can be found by assigning a random guess $s_1$ and recursively computing $\{s_n\}_{n>1}$ until convergence using:
\begin{align}
    &s_{n+1} =\frac{\sum_{x\in\vec{t}}\left[|x|\cdot \mathbbm{1}_{\{|x|>s_n\}} \right]}{\frac{4^{-B}}{3}\sum_{x\in\vec{t}}\left[\mathbbm{1}_{\{0<|x|\leq s_n\}}\right] +\sum_{x\in\vec{t}}\left[\mathbbm{1}_{\{|x|>s_n\}}\right]}
\label{eqn:octav_corr}
\end{align}
\end{corollary}
\begin{proof}
The empirical distribution of the data inside $\vec{t}$ is used in lieu of the abstract $f_X()$ in Theorem \ref{thm:octav}. Thus, expectations in \eqref{eqn:octav_theorem} are replaced by average summations. Numerator and denominator are both multiplied by the number of elements in $\vec{t}$, suppressing the need for division. As zeros can be represented using integer quantization, zero elements in $\vec{t}$ are excluded from the distribution (see first term in the denominator). This is done to prevent an over-estimation of the total quantization noise for very sparse tensors. 
\end{proof}
We call the algorithm in Corollary \ref{corr:octav} Optimally Clipped Tensors And Vectors (OCTAV). Derived from the Newton-Raphson method, OCTAV converges very quickly algorithmically. We only use 10 iterations for any OCTAV implementation in this paper. Additionally, OCTAV is insensitive to the choice of initial guess. Indeed, for weight and activation tensors of a pretrained ResNet-50 network, OCTAV consistently converges to the same solution for various choices of $s_1$ including $0$, $s_{\max}$, $3\sigma_{\vec{t}}$, $4\sigma_{\vec{t}}$, and $5\sigma_{\vec{t}}$; where $\sigma_{\vec{t}}$ denotes the tensor standard deviation. In our upcoming experiments, we use $s_1 = \sum_{x\in\vec{t}}|x|/\sum_{x\in\vec{t}}\mathbbm{1}_{|x|>0}$, which is the value of $s_3$ if the initial guess is set to $s_{\max}$.

Computationally, each iteration of the OCTAV algorithm can be implemented using fast operations. Indeed, the only vector/tensor operations required are the indicator function, which is realizable via simple Boolean datatype casting, and element-wise absolute values, multiplications, and comparisons. Afterwards, sum reductions are performed and only residual scalar operations remain, including one division.

The algorithmic and computational efficiencies of OCTAV make it significantly faster than a conventional brute force search for $s^*$. On a CPU, and with no code optimizations, OCTAV is $\sim$10$\times$ faster than brute force when applied to weight and activation tensors of a BERT-Base model. Details of this comparison are included in Appendix \ref{appendix:timing_comparison}. 

Importantly, all operations required in \eqref{eqn:octav_corr} are tensor operations. Thus, OCTAV can be implemented on GPUs using any deep learning package. For instance, our implementation only invokes native PyTorch \cite{pytorch} operations. Consequently, we can embed OCTAV into any QAT routine to realize dynamic quantization using optimal clipping scalars for each tensor at each iteration. The added optimization does incur an overhead, but because OCTAV is fast, it is possible to perform the desired QAT in reasonable amounts of time. We also note that all OCTAV operations are broadcastable and can be used when sub-tensor scaling is required \cite{wu2020integer}. Thanks to the broadcasts, optimization for finer-grained scaling incurs no slowdown.

\begin{figure}[!t]
\begin{center}
    \begin{subfigure}[t]{0.49\linewidth}
    \centering
        \includegraphics[trim=7cm 0cm 7cm 0cm, clip, width = 0.99\linewidth,page=5]{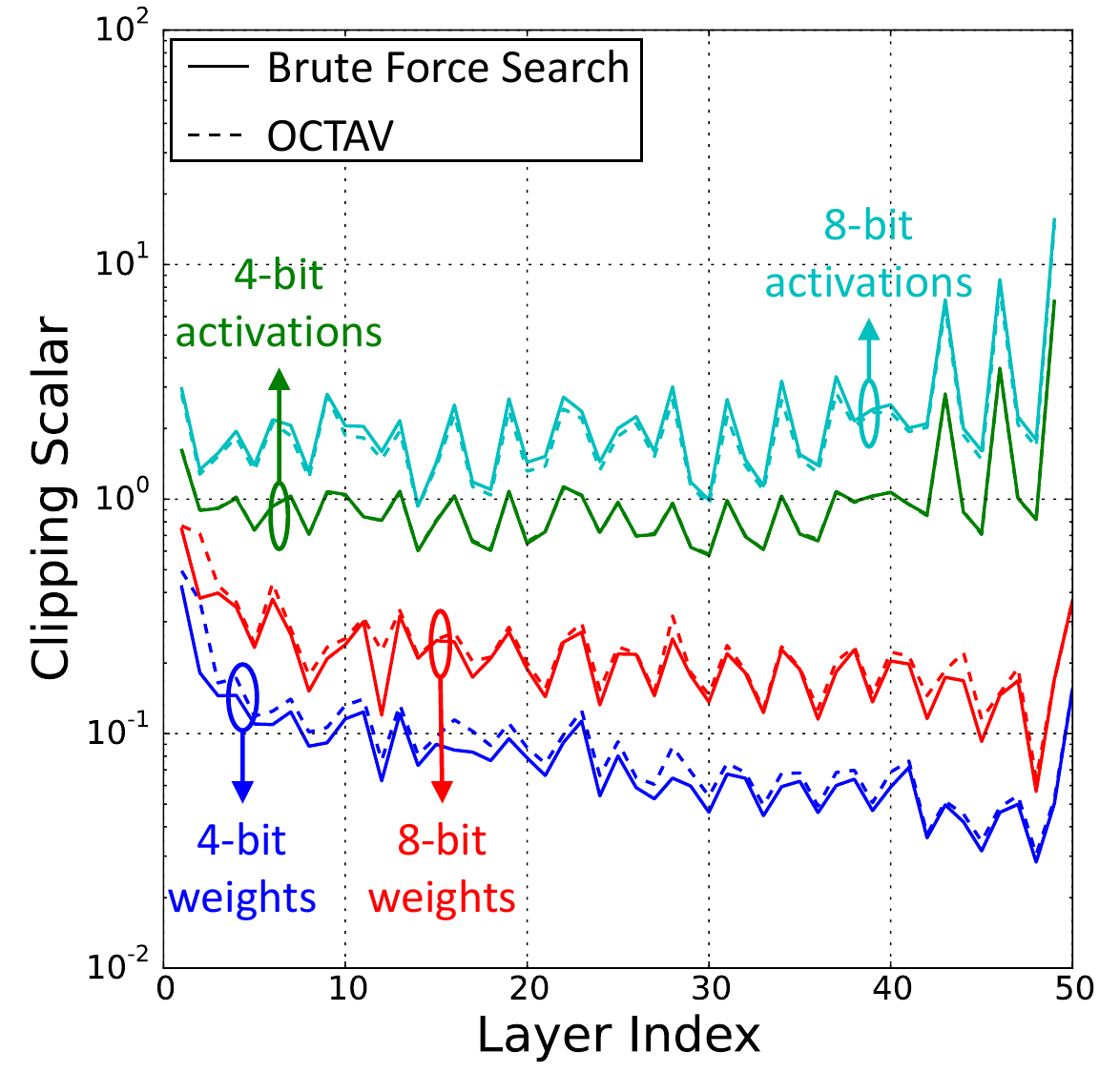}
    \caption{}
    \end{subfigure}
    \begin{subfigure}[t]{0.49\linewidth}
    \centering
        \includegraphics[trim=7cm 0cm 7cm 0cm, clip, width = 0.99\linewidth,page=6]{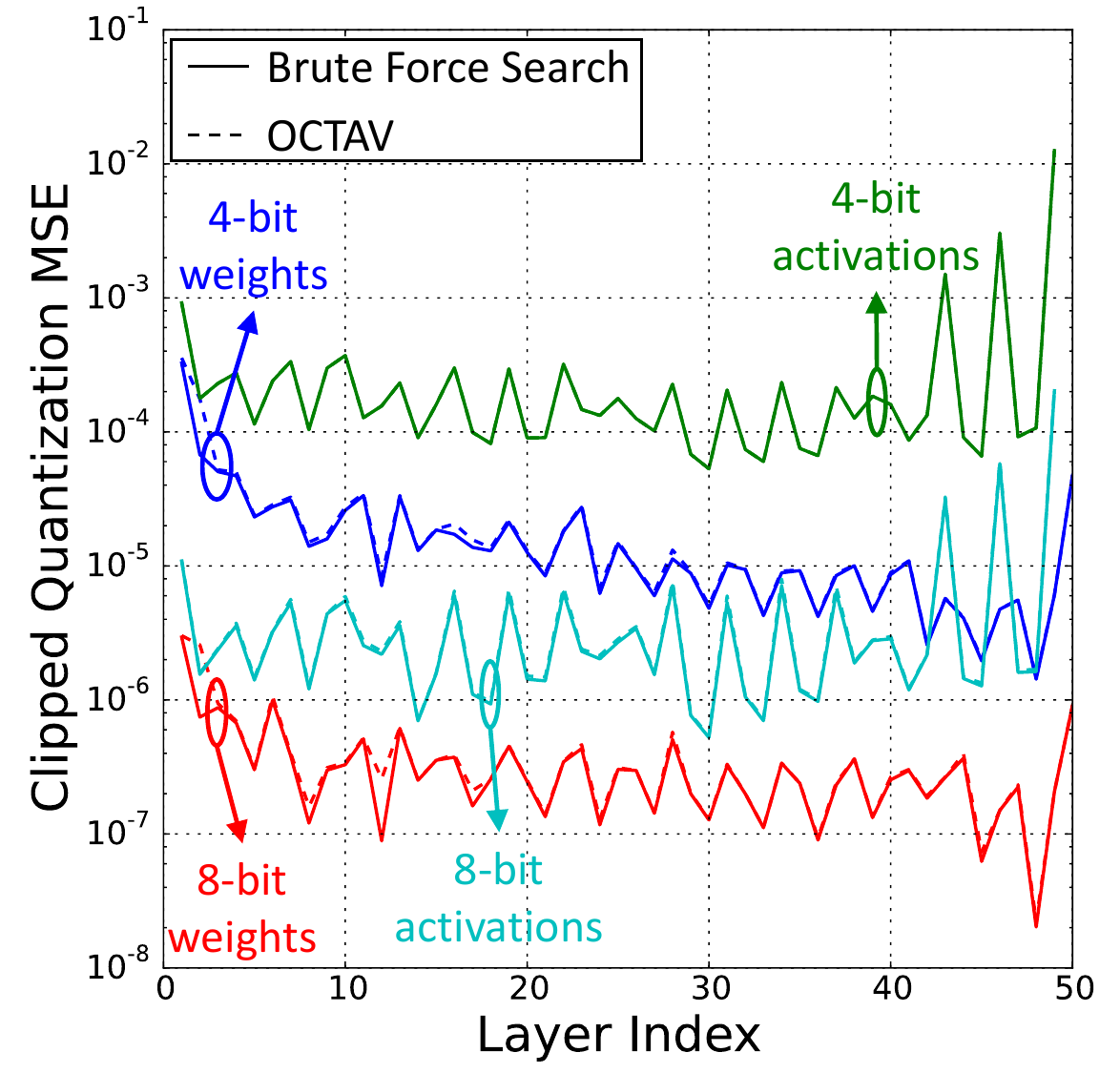}
    \caption{}
    \end{subfigure}
\end{center}
\caption{Comparison of (a) optimal clipping scalar as determined by a brute force search (solid lines) and OCTAV (dashed lines), and (b) corresponding empirically measured clipped quantization MSE at every weight and activation layer in a pretrained ResNet-50 model for $B=4,8$. The OCTAV-determined clipping scalar is the result of invoking Corollary \ref{corr:octav} and \eqref{eqn:octav_corr} while the brute force search is realized by sweeping the value and $s$ and empirically evaluating \eqref{eqn:mse_definition}.}
\label{fig:octav_sweep_comparison_resnet}
\end{figure}

The OCTAV algorithm is guaranteed to converge to the global optimum of the convex MSE $J(s)$ in \eqref{eqn:clipped_quantization_mse}. The trade-off between clipping and discretization noise discussed in Section \ref{sec:clipped_quantization} leads to this convexity, which is verified by virtue of the second derivative $J''(s)$ being positive (see Appendix \ref{appendix:octav_proof}). In Figure \ref{fig:octav_sweep_comparison_resnet}(a), we plot the optimal clipping scalar for all weight and activation layers in a ResNet-50 pretrained model, as determined by OCTAV and a brute force search. Consistently, both solutions are either equal or close to one another. Even in the case of a slight mismatch, the resulting quantizers have identical MSE, as shown in Figure \ref{fig:octav_sweep_comparison_resnet}(b). 

Recall that \eqref{eqn:clipped_quantization_mse} is derived with additive noise assumed in the discretization region. This model is valid provided the quantization step is small \cite{widrowBook}. On rare occasions, and in the presence of very large outliers, this model can be inaccurate. Indeed, using a large clipping scalar to cater for outliers at the expense of quantizing all small values to zero leads to one local minimum of the empirical MSE in \eqref{eqn:mse_definition} unidentified by \eqref{eqn:clipped_quantization_mse}. In this case, OCTAV still converges to the local minimum balancing discretization and clipping, where the noise model in \eqref{eqn:clipped_quantization_mse} is valid. This rare phenomenon was observed in some activation layers of BERT models and investigated in Appendix \ref{appendix:mse_convexity}. Interestingly, the local minimum to which OCTAV converges to is a favorable one for QAT, as shown in Section \ref{sec:fine_tuning}.

{The formulation above minimizes the quantization MSE by choice. We aspire to train with minimum quantization noise variance, and have shown above how to do so using OCTAV. Minimizing noise is a desirable feature of QAT, and our promising experimental results in Section \ref{sec:experiments} support this contention. Using the above idea of online optimization using the Newton-Raphson algorithm, it is possible to derive similar methods for minimizing alternative quantization fidelity metrics such as $L_p$-norm, KL divergence, and others. Such optimizations are beyond the scope of this paper, but can form the basis of interesting extensions of our work.}

\section{Improving QAT Gradient Estimation}
\label{sec:gradients}
The OCTAV algorithm enables QAT with minimal noise at each iteration, thereby boosting accuracy. Nonetheless, QAT fundamentally relies on differentiating the discontinuous quantization operation, requiring a gradient estimator. The estimation choice impacts convergence, warranting an analysis of available options. We present mathematical limitations of the commonly employed straight-through estimator (STE) and piece-wise linear (PWL) gradients for clipped quantization. 
We then overcome these limitations by proposing the magnitude-aware derivative (MAD).

\subsection{Limitations of Current Gradient Estimation}
\begin{figure}[!t]
\begin{center}
        \includegraphics[trim=6.5cm 2.5cm 8cm 2cm, clip, width = 0.67\linewidth,page=7]{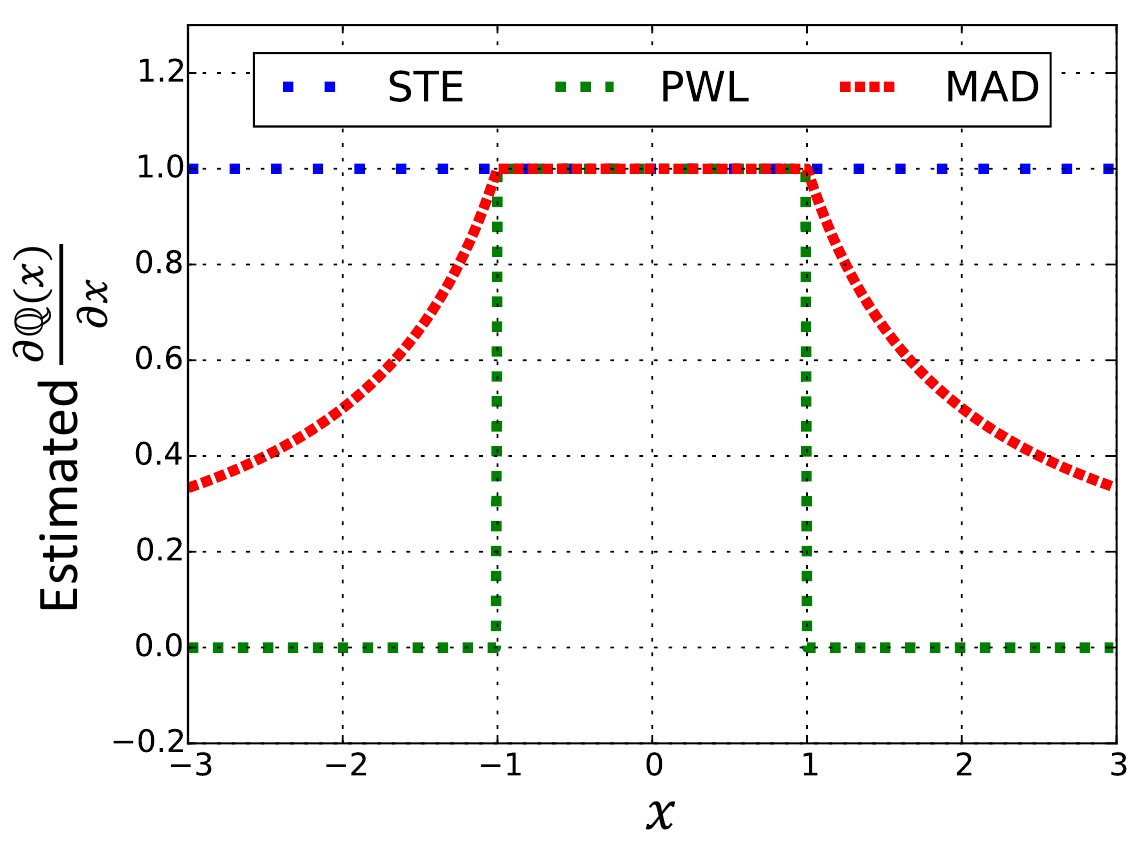}
\end{center}
\caption{Gradient estimators for the clipped quantization operation including STE, PWL, and our proposed MAD. For this example, we assume the clipping scalar to be $s=1$.}
\label{fig:derivatives_comparisons}
\end{figure}
We start with an analysis of gradient back-propagation using the STE that overlooks \eqref{eqn:clipped_quantization_definition} to set $\frac{\partial^{\text{(STE)}}\mathbbm{Q}(x)}{x} = 1$. With clipping, this approximation results in gradient explosion, which causes instability. This is shown via a second order (variance) study similar to that of \cite{heInit} for initialization. Such analysis is useful for assessing the suitability of back-propagation with quantization \cite{sakrICLRfl}.

For an arbitrary activation $x_l$ at layer $l$, we write $\Delta x_l=\frac{\partial\mathcal{L}}{\partial x_l}$ to be the true gradient with respect to the loss function $\mathcal{L}$. This gradient is fundamentally defined as the rate of marginal change in loss function for a marginal change in activation value. Further, let $\Delta^{\text{(STE)}}x_l$ be the estimate of $\Delta x_l$ under STE. The following result holds.
\begin{proposition}[Gradient explosion with STE]
\label{prop:grad_explosion}
In an $L$-layer network, there exists a positive $\delta$ such that the ratio of variances of STE gradient $\text{Var}\left(\Delta^{\text{(STE)}}X_l \right)$ to true gradient $\text{Var}\left(\Delta X_l \right)$ at layer $l$ is lower bounded by:
\begin{align}
    \label{eqn:grad_explosion}
    \frac{\text{Var}\left(\Delta^{\text{(STE)}}X_l \right)}{\text{Var}\left(\Delta X_l \right)} \geq (1+\delta)^{L-l}
\end{align}
\end{proposition}
\begin{proof}
The proof is provided in Appendix \ref{appendix:grad_explosion_proof}. The main insight is that STE carries excess variance due to its assigning unity to gradients of clipped weight (see Figure \ref{fig:derivatives_comparisons}).
\end{proof}

The result in Proposition \ref{prop:grad_explosion} highlights an exponential explosion of back-propagated STE gradients. In contrast,
the PWL estimator sets $\frac{\partial^{\text{(PWL)}}\mathbbm{Q}(x)}{x} = \mathbbm{1}_{x\in[-s,s]}$ and does not suffer from such gradient explosion. However, weight tensors trained using PWL encounter a partial stoppage of convergence, as early as the first training iteration. Early stopping is equivalent to model size reduction, which can impede the achievable accuracy. The following result holds.

\begin{proposition}[Convergence stoppage with PWL]
\label{prop:pwl_stoppage}
Given a statically clipped $N_{\vec{w}}$-element weight tensor $\vec{w}$, whose gradient is estimated using PWL, only $\tilde{N}^{(i)}_{\vec{w}}$ of its parameters are leaned at iteration $i$, and the following inequalities hold:
\begin{align}
    \label{eqn:pwl_stoppage}
    N_{\vec{w}} > \tilde{N}^{(i)}_{\vec{w}} \geq \tilde{N}^{(i+1)}_{\vec{w}}
\end{align}
\end{proposition}
\begin{proof}
The proof is provided in Appendix \ref{appendix:pwl_stoppage_proof}. The main insight is that PWL repeatedly zeroes out gradients of clipped weights (see Figure \ref{fig:derivatives_comparisons}), halting their updates.
\end{proof}
The monotonic decrease in Proposition \ref{prop:pwl_stoppage} requires static quantization. Nevertheless, dynamic quantization exhibits a similar, albeit milder, convergence stoppage, where the first strict inequality in \eqref{eqn:pwl_stoppage} also holds.



\subsection{Magnitude-aware Differentiation}
\label{sec:mad}
To formulate an improved gradient estimator, we first present a simple result: rather than treating clipping as a piece-wise selection, we write it as a \textit{magnitude attenuation}.
\begin{proposition}
\label{prop:mad_clipping}
The clipping operator is given by:
\begin{align}
    \label{eqn:mad_clipping}
    \text{clip}(x,-s,s) = \alpha\cdot x
\end{align}
where $\alpha = \mathbbm{1}_{\{|x|\leq s\}} + \frac{s}{|x|}\mathbbm{1}_{\{|x|>s\}}$  
\end{proposition}
\begin{proof}
The result can readily be obtained by replacing the indicator function by its definition, i.e.: \\
$\mathbbm{1}_{\{|x|\leq s\}} = (1-\mathbbm{1}_{\{|x|>s\}}) = \begin{cases}1 \quad \text{if } |x|\leq s\\
    0 \quad \text{if } |x|>s
\end{cases}$
\end{proof}
Using Proposition \ref{prop:mad_clipping}, we formulate the magnitude-aware derivative (MAD). Treating $\alpha$ as a constant in \eqref{eqn:mad_clipping}, we obtain:
\begin{align}
    \label{eqn:mad_derivative}
    \frac{\partial^{\text{(MAD)}}\mathbbm{Q}(x)}{x} =  \mathbbm{1}_{\{|x| \leq s\}} + \frac{s}{|x|}\mathbbm{1}_{\{|x|>s\}}
\end{align}

In Figure \ref{fig:derivatives_comparisons}, we plot the three gradient estimators: STE, PWL, and MAD, for $s=1$. They are identical in the discretization region. However, while PWL zeroes out the clipping region, MAD uses a magnitude-aware attenuation factor and is  continuous. Therefore, for a MAD-trained weight tensor $\vec{w}$, we do guarantee that $\tilde{N}^{(i)}_{\vec{w}} = N_{\vec{w}}$ at any iteration $i$, and there is no early stoppage of convergence. 

In some measure, PWL and MAD are similar. The former approximates $\frac{\partial \mathbbm{1}_{\{|x|\leq s\}}}{\partial s} = \frac{\partial \mathbbm{1}_{\{|x|>s\}}}{\partial s} = 0$. This style of approximation is predominant and useful in deep learning, e.g., it is used to train networks with ReLU-like activation functions. Similarly, MAD approximates a combination of indicator functions as being a constant to obtain a useful gradient estimator that evades the limitations of PWL.

Above, we have shown that MAD improves differentiation for quantized weights. There is no clear advantage to using MAD over PWL for activations. The only difference between the two is the occasional zeroing out of activation gradients under PWL. We argue that this is in fact desirable for potential regularization. Indeed, it mimics Dropout \cite{dropout} in the backward path. Thus, we generally recommend using MAD for weight gradients and PWL for activation gradients. In our experiments, we term such a combination MAD-PWL Hybrid (MPH).

{Finally, we note that MAD is different from  \emph{magnitude-aware} gradients in Bi-Real Net \cite{birealnet}, where a triangular pulse combined with an approximate sign operator estimates gradients for \emph{binarization only}. In contrast, and orthogonally, we formulate a derivative to clipped quantization by analyzing its magnitude attenuation effect.}

\section{Quantization-aware Training Studies}
\label{sec:experiments}
\begin{figure}[!t]
\begin{center}
    \includegraphics[width = 0.99\linewidth]{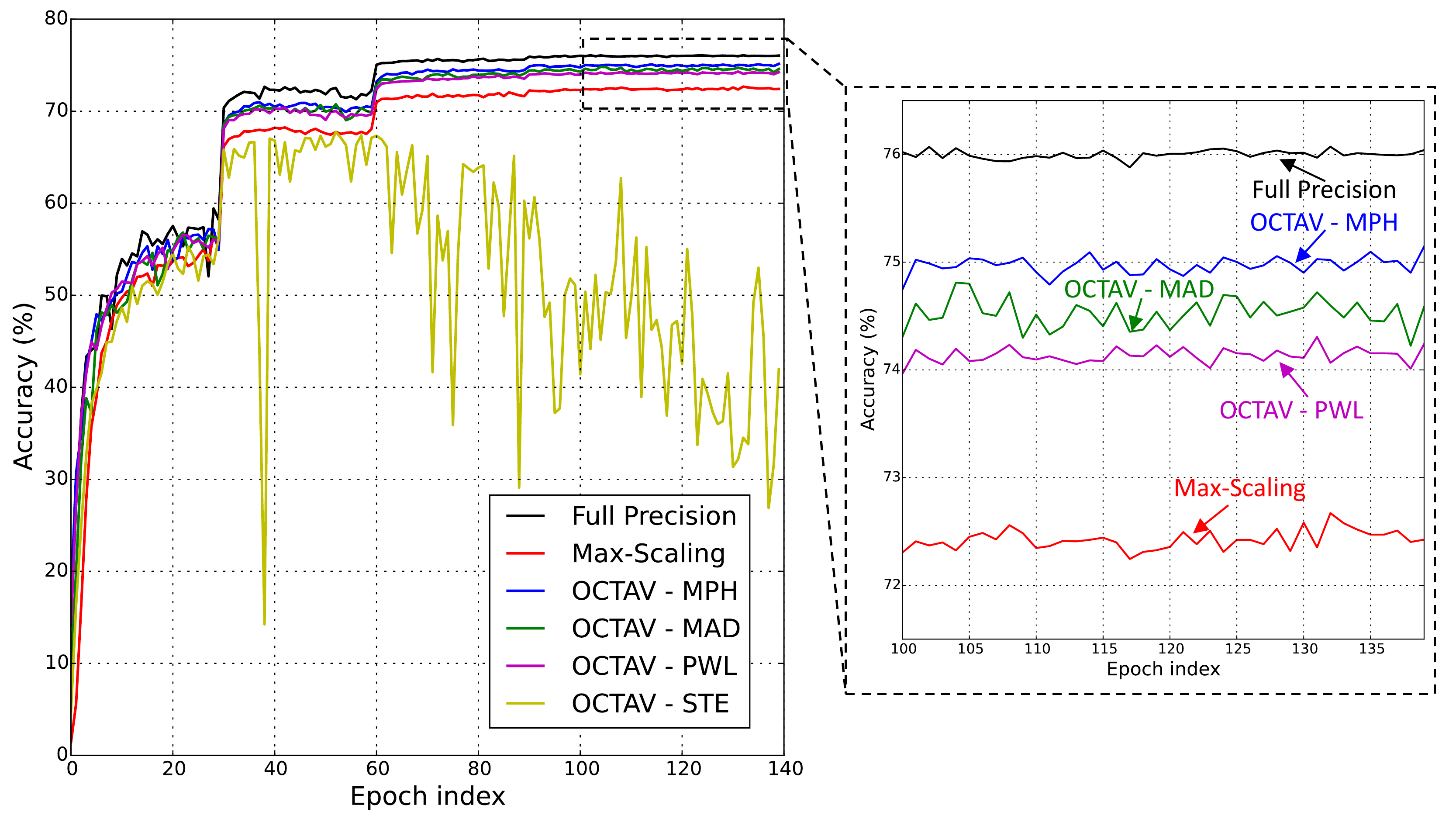}
\end{center}
\caption{Convergence curves for 4-bit ResNet-50 training-from-scratch including full precision baseline, max-scaling, and OCTAV-enabled QAT. For the last, various gradient estimation strategies discussed in Section \ref{sec:gradients} are included: STE, PWL, MAD, and MPH. }
\label{fig:resnet_50_convergence}
\end{figure}
\begin{table}[!t]
    \centering
    \small
    \caption{4-bit ResNet-50 training-from-scratch accuracy}
    \label{table:resnet50_gradient_ablation}
    
    \begin{tabular}{|c|c|c|c|c|c|}
    \hline
    \multirow{2}{*}{\specialcell{Full \\Precision}} & \multirow{2}{*}{\specialcell{Max-\\Scaling}} & \multicolumn{4}{c|}{OCTAV}\\
    \cline{3-6}
    & & STE&PWL&MAD&MPH\\
    \hline
    76.07&72.67&67.75&74.31&74.81&\textbf{75.15}\\ \hline
    \end{tabular}
\end{table}

\begin{table*}[!t]
    \centering
    \small
    \caption{Accuracies for training-from-scratch QAT on ImageNet} 
    \label{table:qat_from_scratch}
    \begin{tabular}{|c||c||c|c||c|c||c|c|}
    \hline
        \multirow{2}{*}{Network} & \multirow{2}{*}{\specialcell{Full Precision\\ Baseline}} & \multicolumn{2}{c||}{$B=4$}&\multicolumn{2}{c||}{$B=6$}&\multicolumn{2}{c|}{$B=8$} \\
        \cline{3-8}
        & & \textbf{OCTAV} & Max-Scaling & \textbf{OCTAV} & Max-Scaling & \textbf{OCTAV} & Max-Scaling\\
        \hline \hline
        ResNet-50 & 76.07 & \textbf{75.15} & 72.67 & 76.07 & 76.01 & 76.24 & 76.12 \\
        \hline
        ResNet-18 & 70.12 & \textbf{69.17} & 65.65 & 69.78 & 69.52 & 70.07 & 70.19 \\
        \hline
        ResNet-101 & 77.28 & \textbf{76.48} & 72.53 & 77.30 & 77.04 & 77.31 & 77.15\\
        \hline
        MobileNet-V2 & 71.71 & \textbf{70.88} & 69.17 & 71.64 & 71.79 & 71.71 & 71.77 \\
        \hline
        MobileNet-V3-Small & 65.99 & 54.68 & 0.39 & \textbf{65.02} & 60.17 & 65.98 & 65.14 \\
        \hline
        MobileNet-V3-Large & 72.97 & 65.86 & 1.25 & \textbf{72.12} & 69.38 & 72.89 & 72.78\\
        \hline
    \end{tabular}
\end{table*}
We conduct numerical experiments to show the impact of our proposed methods. We evaluate training-from-scratch and retraining QAT using ResNet \cite{resnet} and MobileNet \cite{mobilenetv2,mobilenetv3} models deployed on the ImageNet \cite{imagenet} dataset for image classification. For fine-tuning QAT, we use BERT \cite{bert} language models pretrained on the Wikipedia \cite{wikipedia} and BookCorpus \cite{bookcorpus} datasets and fine-tuned on Squad v1.1 \cite{squad} for question-answering.
Our implementations are derived from the NVIDIA `Deep Learning Examples' 
repository \footnote{Code retrieved from: \url{https://github.com/NVIDIA/DeepLearningExamples}.}. All details are in Appendix \ref{appendix:implementations_details}.


\subsection{Training-from-scratch QAT on ImageNet}
\label{sec:training_from_scratch}

To get started, we single out 4-bit ResNet-50 training-from-scratch to point out various aspects of our results. In Figure \ref{fig:resnet_50_convergence}, we compare convergence of test accuracy for the full precision baseline, max-scaled QAT, and OCTAV-enabled QAT, including gradient estimation options from Section \ref{sec:gradients}. 

Conforming to the gradient explosion prediction in Proposition \ref{prop:grad_explosion}, using the STE to differentiate the clipped quantization operation leads to a clear training instability. Furthermore, using PWL is clearly inferior to MAD, which confirms our analysis on early stoppage of convergence for the former and justifies our proposal for the latter. We obtained a marginal improvement in accuracy over MAD with the MPH scheme, due to its better regularization property described at the end of Section \ref{sec:mad}. All further clipped QAT results use MPH.

Table \ref{table:resnet50_gradient_ablation} summarizes ResNet-50 results and lists achieved accuracy for various schemes considered. We find that OCTAV-enabled QAT improves on max-scaling by $\sim$2.5$\%$ and achieves a less-than-1\% accuracy drop compared to the full precision baseline. This by itself matches the current state-of-the-art in 4-bit QAT \cite{pact}. 
We emphasize the significance of this result by recalling that no modification to the training recipe was required, such as adding learned parameters or hyperparameter tuning.

Additional training-from-scratch experiments are reported in Table \ref{table:qat_from_scratch} for various networks and precisions. We consider 4-bit, 6-bit, and 8-bit QAT for: ResNet-50, ResNet-18, ResNet-101, MobileNet-V2, MobileNet-V3-Small, and MobileNet-V3-Large. We highlight results yielding high accuracy at low-precision. 

For ResNets, 6-bit is enough to achieve close-to-baseline accuracy, even with max-scaling. However, at 4-bit, OCTAV is required to remain within 1\% of the baseline while max-scaling leads to a drop in accuracy of up to $\sim$5$\%$.

Similar trends are observed for MobileNet-V2; we find this network easier to quantize compared to MobileNet-V3. 
We speculate that this is due to the former using ReLU6 activations as opposed to the latter, which uses the Hardswish function \cite{mobilenetv3}.

Max-scaled QAT of MobileNet-V3 leads to a large drop in accuracy of $\sim$ 3-to-5 \% at 6-bit, and further quantization to 4-bit {fails to converge} and results in near zero accuracy.
In contrast, OCTAV-enabled QAT of MobileNet-V3 is successful at 6-bit. A noticeable, but not total, accuracy drop at 4-bit leaves room for improvement. We speculate that accuracy can be recovered through a QAT-friendly training recipe, such as distillation \cite{profitECCV2020}, but this is beyond the scope of our work.

\subsection{Retraining ImageNet networks at 4-bit}
\label{sec:re_training}

\begin{table*}[!t]
    \centering
    \small
    \caption{Accuracies for {short} 4-bit retraining QAT on ImageNet}
    \label{table:retraining_qat}
    \begin{tabular}{|c||c|c||c|c|c|c|c|}
    \hline
        \multirow{2}{*}{Network} & \multicolumn{2}{c||}{Dynamic Quantization} & \multicolumn{5}{c|}{Static Quantization} \\
        \cline{2-8}
        & \textbf{OCTAV} & Max-Scaling & \textbf{OCTAV} & MSE Sweep & $99.9^{\text{th}}$ Perc. & $99.99^{\text{th}}$ Perc. & $99.999^{\text{th}}$ Perc.\\
        \hline \hline
        ResNet-50 & \textbf{75.38} & 71.44 & \textbf{75.84} & 75.85 & 75.66 & 75.51 & 75.29 \\
        \hline
        ResNet-18 & \textbf{69.16} & 65.53 & \textbf{69.18} & 69.28 & 69.04 & 69.08 & 68.93 \\
        \hline
        ResNet-101 & \textbf{76.10} & 70.88 & \textbf{76.96} & 77.01 & 76.79 & 76.99 & 76.34\\
        \hline
        MobileNet-V2 & \textbf{69.32} & 66.94 & 0.66 & 0.93 & 1.72 & 2.14 & 2.76 \\
        \hline
        MobileNet-V3-Small & \textbf{53.52} & 0.10 & 0.43 & 0.58 & 0.65 & 0.10 & 1.46 \\
        \hline
        MobileNet-V3-Large & \textbf{64.97} & 39.46 & 0.39 & 0.30 & 0.34 & 0.71 & 0.57\\
        \hline
    \end{tabular}
\end{table*}

\begin{table}[!t]
    \centering
    \small
    \caption{{Accuracies for long 4-bit retraining QAT on ImageNet using OCTAV}}
    \label{table:retraining_qat_full}
    \begin{tabular}{|c||c|c|}
    \hline
        Network & Dynamic-OCTAV & Static-OCTAV \\
        \hline \hline
        ResNet-50 & \textbf{76.21} & \textbf{76.46} \\
        \hline
        ResNet-18 & \textbf{69.90} & \textbf{70.13} \\
        \hline
        ResNet-101 & \textbf{76.84} & \textbf{77.48} \\
        \hline
        MobileNet-V2 & \textbf{71.23} & 1.21 \\
        \hline
        MobileNet-V3-Small & \textbf{58.93} & 0.80 \\
        \hline
        MobileNet-V3-Large & \textbf{69.21} & 0.60 \\
        \hline
    \end{tabular}
\end{table}
We also study 4-bit retraining of the aforementioned networks. We use a \emph{shortened} version of the same training recipe, with details included in Appendix \ref{appendix:implementations_details}. We include static quantization in our results and follow similar methods as \cite{wu2020integer}. We report results for $99.9^{\text{th}}$, $99.99^{\text{th}}$, and $99.999^{\text{th}}$ percentile calibration, which were claimed to work well in the retraining setup. {For these experiments, Resnet-50 and ResNet-101 are retrained for 15 epochs, while other networks are retrained for 30 epochs.}

We also include static-OCTAV, which calibrates clipping scalars using \eqref{eqn:octav_corr}. As calibration can be performed offline, we also include results when using a 100-point brute force MSE sweep to set the clipping scalars.

Retraining results are listed in Table \ref{table:retraining_qat} and promising OCTAV results are highlighted. A clear trend is observed. Large models, such as ResNets, are much easier to retrain compared to small models, such as MobileNets. This finding is consistent with recent literature \cite{dboukEECV2020} and leads to several novel conclusions.

For large models, static quantization is most  suitable, though OCTAV-enabled dynamic quantization also yields high accuracy, within $\sim$1$\%$ of the baseline. As large models are easy to quantize, we speculate that a low-precision solution exists near the pretrained starting point. When an aspect of the parameters (the clipping scalars) is forced to be static, the model rapidly settles around a close solution to this starting point. In the case of ResNets, this solution is highly accurate. Highest accuracy is consistently achieved using static-OCTAV and the MSE sweep, both similar as expected. 

For small models, such as MobileNets, static quantization is unfit, yielding near-zero accuracy. For such models, a good quantized solution is unlikely to be close to the pretrained starting point. Thus, retraining requires dynamic quantization to track changes in tensor distributions occurring as the model adapts to low precision. Furthermore, OCTAV is found to be far superior to other strategies and can recover some accuracy across all MobileNets.

Interestingly, we find that retraining can be less accurate than training-from-scratch. For MobileNet-V2, OCTAV reaches  69.32\% for the former and 70.88\% for the latter. However, this is due to the shorter training time of 30 epochs used in retraining. When equalizing QAT time and retraining for 300 epochs using OCTAV, we obtain an accuracy of \textbf{71.23}\%. Thus, using a pretrained starting-point has some merits in spite of the large amounts of quantization noise suffered by MobileNets at low precision. 

{To further explore the potential of OCTAV and to understand the merits of retraining, we also perform \emph{long} 4-bit retraining of all ImageNet models considered. We retrain with OCTAV both dynamically and using static calibration, following the above setup. The full long retraining recipes are identical to those employed for training-from-scratch experiments in Section \ref{sec:experiments}.1 but for two aspects. First, the starting point is the pretrained model in full precision; and, second, the starting learning rate value is attenuated by a factor of $10^{-2}$ compared to that of the training-from-scratch recipe. For these experiments, ResNets and MobileNets are retrained for 150 and 300 epochs, respectively.}

{Accuracies for long 4-bit retraining QAT are reported in Table \ref{table:retraining_qat_full} and promising results are highlighted. Compared to results in Table \ref{table:retraining_qat}, long retraining always leads to a noticeable improvement, which establishes the merits of prolonged training. Consistent with earlier observations, static-OCTAV is found to yield highest accuracy for ResNets, but suffers catastrophic degradation on MobileNets. Similarly, dynamic-OCTAV once more yields high accuracy for all networks.}

{For ResNets, the achieved accuracy is either close to, or superior to that of the full precision baseline. This indicates that OCTAV's quantization effects are so small and fundamentally fall below the training algorithm's inherent noise floor. Our results are comparable in absolute terms to the state-of-the-art reported accuracies for 4-bit QAT on ResNets established using learned quantization such as PACT \cite{pact} and LSQ \cite{LSQ}. Comparing relative accuracy to the starting full precision baseline, our results improve the state-of-the-art. For instance, for ResNet-50, 4-bit OCTAV improves the baseline accuracy by $\sim$0.4$\%$ while 4-bit LSQ degrades it by $\sim$0.2$\%$.}

{Trends of long MoblileNet retraining are similar to those previously observed for training-from-scratch and short retraining. Failure of static calibration, in spite of prolonged training time, illustrates the importance of tracking tensor statistics when retraining small models. The improvement in final accuracy under dynamic-OCTAV is promising but not enough in the case of Mobilenet-V3. As mentioned in Section \ref{sec:experiments}.1, further QAT techniques may be required to achieve close-to-baseline accuracy.}

\subsection{Fine-tuning QAT of BERT Models on Squad}
\label{sec:fine_tuning}

\begin{table*}[!t]
    \centering
    \small
    \caption{Accuracies for fine-tuning BERT-Base \& BERT-Large on Squad v1.1}
    \label{table:bert_combined}
    \begin{tabular}{|c|c||c|c|c|c|c||c|c|c|c|c|}
    \hline
    \multicolumn{2}{|c||}{Network} & \multicolumn{5}{c||}{BERT-Large (Baseline Accuracy: 91.00)} & \multicolumn{5}{c|}{BERT-Base (Baseline Accuracy: 88.24)}\\
    \hline 
    \multicolumn{2}{|c||}{\# of bits $B$} & 4& 5& 6& 7& 8& 4& 5& 6& 7& 8\\
    \hline \hline
    \multirow{2}{*}{\specialcell{Dynamic\\ Quantization}}& \textbf{OCTAV} & \textbf{87.09}& \textbf{89.77}& \textbf{90.51}& \textbf{90.81}& \textbf{90.78}& \textbf{84.51}& \textbf{86.30}& \textbf{87.43}& \textbf{88.28}& \textbf{88.34} \\
    \cline{2-12}
    &Max-Scaling &6.92& 80.06& 87.71& 90.04& 90.48& 11.51& 78.97& 85.17& 87.46& 88.01\\
    \hline \hline
    \multirow{5}{*}{\specialcell{Static\\ Quantization }}& \textbf{OCTAV} &\textbf{87.08}& \textbf{89.54}& \textbf{90.60}& \textbf{90.79}& \textbf{90.61}& \textbf{83.60}& \textbf{85.82}& \textbf{87.14}& \textbf{87.67}& \textbf{88.02} \\
    \cline{2-12}
    &MSE Sweep & 85.54& 89.77& 90.39& 90.80& 90.55& 81.82& 84.16& 87.14& 87.68& 87.97\\
    \cline{2-12}
    &$99.9^{\text{th}}$ Perc. &86.98& 89.79& 89.99& 90.07& 90.11& 81.06& 85.78& 86.73& 86.84& 87.34\\
    \cline{2-12}
    &$99.99^{\text{th}}$ Perc. &6.90& 87.63& 90.38& 90.79& 90.33& 67.90& 83.20& 86.78& 87.60& 87.94\\
    \cline{2-12}
    &$99.999^{\text{th}}$ Perc. &4.56& 5.66& 89.76& 90.44& 90.83& 26.85& 82.15& 86.27& 87.51& 88.08\\
    \hline
    \end{tabular}
\end{table*}
Finally, we study fine-tuning QAT of BERT-Base and BERT-Large on Squad v1.1. As research on low-precision transformer networks is still in the early stages, we study QAT using every bit-width from 4 to 8 bits to understand the difficulties in quantizing these networks. 

We employ the same strategies of dynamic and static quantization as were used for retraining in Section \ref{sec:re_training}. For BERT-Base, we compare calibration times on a CPU when using OCTAV and the MSE sweep. Details are included in Appendix \ref{appendix:timing_comparison} and the former is found to be $\sim$10$\times$ faster.

Fine-tuning results are included in Table \ref{table:bert_combined} where various F1 scores are reported, and OCTAV results are highlighted.

For both networks, 7-bit or more is enough to match the baseline accuracy, regardless of the quantization strategy. At 6-bit, only dynamic-OCTAV yields an accuracy within 1\% of the baseline for the two networks. At lower precision, dynamic-OCTAV exhibits the most graceful degradation in accuracy, reaching a drop of $\sim$4$\%$ for both networks at 4-bit. As fine-tuning consists of re-adapting the model to a new task, it is reasonable to expect changes in tensor statistics, warranting the use of dynamic quantization for tracking. Thus, dynamic-OCTAV being consistently superior to all other strategies is expected.

Static quantization performs generally well, and static-OCTAV is clearly its best candidate. Its accuracy gracefully degrades to a 4-to-5 \% drop compared to the baseline. A surprising result is the definite superiority of static-OCTAV compared to the MSE sweep. It turns out that, for a few activation layers, the clipped quantization MSE is not convex, leading to divergent solutions for the two strategies. This issue occurs due to the presence of large outliers. The local minimum closest to zero is selected by OCTAV while the MSE sweep chooses to zero out all small values. This phenomenon is described in Appendix \ref{appendix:mse_convexity}. The better choice made by OCTAV  leads to $\sim$1.5$\%$ better accuracy compared to the sweep at 4-bit.

\section{Discussion}
\subsection{Current limitations and directions for future work}
We have shown, analytically and empirically, that OCTAV-enabled QAT improves accuracy of low-precision training without requiring modifications to the learning algorithm. As efforts to scale down DNN precision continue, an interesting avenue of future research is to combine OCTAV with quantization-dedicated training recipes, such as distillation, to increase accuracy even further. 

In addition, while the OCTAV overhead is at least an order of magnitude lower than an equivalent sweep, it is still pronounced. Research to reduce this overhead is needed to accelerate FQT in an accuracy-optimal manner. Similarly, for DNN inference acceleration, often hindered by dynamic quantization, techniques to match OCTAV accuracy in the static setup are desired. Our proposed static-OCTAV calibration strategy is one step in that direction.

Finally, the theoretical technique employed in this work may be applied to complexity reduction beyond quantization. We have formulated quantization noise as an objective function to be minimized on the fly using the Newton-Raphson algorithm. Similarly, other hardware-aware models, such as those for sparsification, can be rapidly optimized for reduced complexity. Such work can be impactful in the context of neural architecture search for hardware-efficient DNNs.
\subsection{Conclusion}
We have proposed OCTAV for enabling QAT with minimal quantization noise for each tensor at every iteration. We have also analyzed current clipped gradient estimators and proposed magnitude-aware differentiation as a tool to further improve QAT. Empirically, we have demonstrated that our methods lead to state-of-the-art accuracy in low-precision training of DNNs, without modifying the learning algorithm. Our contributions are an important step in the advancement of low-complexity deep learning.

\subsection*{Acknowledgement}
{The authors wish to thank Ben Keller and Hao Wu from NVIDIA for useful discussions.}

\bibliography{refs}

\begin{thebibliography}{46}
\providecommand{\natexlab}[1]{#1}
\providecommand{\url}[1]{\texttt{#1}}
\expandafter\ifx\csname urlstyle\endcsname\relax
  \providecommand{\doi}[1]{doi: #1}\else
  \providecommand{\doi}{doi: \begingroup \urlstyle{rm}\Url}\fi

\bibitem[Abdolrashidi et~al.(2021)Abdolrashidi, Wang, Agrawal, Malmaud,
  Rybakov, Leichner, and Lew]{paretoCVPR21}
Abdolrashidi, A., Wang, L., Agrawal, S., Malmaud, J., Rybakov, O., Leichner,
  C., and Lew, L.
\newblock Pareto-optimal quantized resnet is mostly 4-bit.
\newblock In \emph{Proceedings of the IEEE/CVF Conference on Computer Vision
  and Pattern Recognition}, pp.\  3091--3099, 2021.

\bibitem[Bianco et~al.(2018)Bianco, Cadene, Celona, and
  Napoletano]{bianco2018benchmark}
Bianco, S., Cadene, R., Celona, L., and Napoletano, P.
\newblock Benchmark analysis of representative deep neural network
  architectures.
\newblock \emph{IEEE Access}, 6:\penalty0 64270--64277, 2018.

\bibitem[Choi et~al.(2018)Choi, Wang, Venkataramani, Chuang, Srinivasan, and
  Gopalakrishnan]{pact}
Choi, J., Wang, Z., Venkataramani, S., Chuang, P. I.-J., Srinivasan, V., and
  Gopalakrishnan, K.
\newblock {PACT}: Parameterized clipping activation for quantized neural
  networks.
\newblock \emph{arXiv preprint arXiv:1805.06085}, 2018.

\bibitem[Choi et~al.(2020)Choi, Choi, El-Khamy, and Lee]{distillationCVPR2020}
Choi, Y., Choi, J., El-Khamy, M., and Lee, J.
\newblock Data-free network quantization with adversarial knowledge
  distillation.
\newblock In \emph{Proceedings of the IEEE/CVF Conference on Computer Vision
  and Pattern Recognition Workshops}, pp.\  710--711, 2020.

\bibitem[Courbariaux et~al.(2015)]{binaryConnect}
Courbariaux, M. et~al.
\newblock Binaryconnect: Training deep neural networks with binary weights
  during propagations.
\newblock In \emph{Advances in Neural Information Processing Systems}, pp.\
  3123--3131, 2015.

\bibitem[Dai et~al.(2021)Dai, Venkatesan, Ren, Zimmer, Dally, and
  Khailany]{daiVSQuant}
Dai, S., Venkatesan, R., Ren, M., Zimmer, B., Dally, W., and Khailany, B.
\newblock Vs-quant: Per-vector scaled quantization for accurate low-precision
  neural network inference.
\newblock \emph{Proceedings of Machine Learning and Systems}, 3, 2021.

\bibitem[Dbouk et~al.(2020)Dbouk, Sanghvi, Mehendale, and
  Shanbhag]{dboukEECV2020}
Dbouk, H., Sanghvi, H., Mehendale, M., and Shanbhag, N.
\newblock Dbq: A differentiable branch quantizer for lightweight deep neural
  networks.
\newblock In \emph{European Conference on Computer Vision}, pp.\  90--106.
  Springer, 2020.

\bibitem[Deng et~al.(2009)Deng, Dong, Socher, Li, Li, and Fei-Fei]{imagenet}
Deng, J., Dong, W., Socher, R., Li, L.-J., Li, K., and Fei-Fei, L.
\newblock {ImageNet}: A large-scale hierarchical image database.
\newblock In \emph{2009 IEEE Conference on Computer Vision and Pattern
  Recognition}, pp.\  248--255. IEEE, 2009.

\bibitem[Devlin et~al.(2018)Devlin, Chang, Lee, and Toutanova]{bert}
Devlin, J., Chang, M.-W., Lee, K., and Toutanova, K.
\newblock Bert: Pre-training of deep bidirectional transformers for language
  understanding.
\newblock \emph{arXiv preprint arXiv:1810.04805}, 2018.

\bibitem[Esser et~al.(2019)Esser, McKinstry, Bablani, Appuswamy, and
  Modha]{LSQ}
Esser, S.~K., McKinstry, J.~L., Bablani, D., Appuswamy, R., and Modha, D.~S.
\newblock Learned step size quantization.
\newblock In \emph{International Conference on Learning Representations}, 2019.

\bibitem[Goel \& Shanbhag(1998)Goel and Shanbhag]{goelQnoise}
Goel, M. and Shanbhag, N.
\newblock {Finite-precision analysis of the pipelined strength-reduced adaptive
  filter}.
\newblock \emph{Signal Processing, IEEE Transactions on}, 46\penalty0
  (6):\penalty0 1763--1769, 1998.

\bibitem[Gonugondla et~al.(2020)Gonugondla, Sakr, Dbouk, and
  Shanbhag]{gonugondlaICCAD}
Gonugondla, S.~K., Sakr, C., Dbouk, H., and Shanbhag, N.~R.
\newblock Fundamental limits on the precision of in-memory architectures.
\newblock In \emph{Proceedings of the 39th International Conference on
  Computer-Aided Design}, pp.\  1--9, 2020.

\bibitem[Gupta et~al.(2015)Gupta, Agrawal, Gopalakrishnan, and
  Narayanan]{ibmIcml15}
Gupta, S., Agrawal, A., Gopalakrishnan, K., and Narayanan, P.
\newblock Deep learning with limited numerical precision.
\newblock In \emph{International Conference on Machine Learning}, pp.\
  1737--1746, 2015.

\bibitem[Han et~al.(2016)Han, Liu, Mao, Pu, Pedram, Horowitz, and
  Dally]{hanEIE}
Han, S., Liu, X., Mao, H., Pu, J., Pedram, A., Horowitz, M.~A., and Dally,
  W.~J.
\newblock Eie: Efficient inference engine on compressed deep neural network.
\newblock \emph{ACM SIGARCH Computer Architecture News}, 44\penalty0
  (3):\penalty0 243--254, 2016.

\bibitem[He et~al.(2015)He, Zhang, Ren, and Sun]{heInit}
He, K., Zhang, X., Ren, S., and Sun, J.
\newblock Delving deep into rectifiers: Surpassing human-level performance on
  {ImageNet} classification.
\newblock In \emph{Proceedings of the IEEE International Conference on Computer
  Vision}, pp.\  1026--1034, 2015.

\bibitem[He et~al.(2016)]{resnet}
He, K. et~al.
\newblock Deep residual learning for image recognition.
\newblock In \emph{Proceedings of the IEEE Conference on Computer Vision and
  Pattern Recognition}, pp.\  770--778, 2016.

\bibitem[Howard et~al.(2019)Howard, Sandler, Chu, Chen, Chen, Tan, Wang, Zhu,
  Pang, Vasudevan, et~al.]{mobilenetv3}
Howard, A., Sandler, M., Chu, G., Chen, L.-C., Chen, B., Tan, M., Wang, W.,
  Zhu, Y., Pang, R., Vasudevan, V., et~al.
\newblock Searching for mobilenetv3.
\newblock In \emph{Proceedings of the IEEE/CVF International Conference on
  Computer Vision}, pp.\  1314--1324, 2019.

\bibitem[Hubara et~al.(2016)]{binaryNet}
Hubara, I. et~al.
\newblock Binarized neural networks.
\newblock In \emph{Advances in Neural Information Processing Systems}, pp.\
  4107--4115, 2016.

\bibitem[Jain et~al.(2019)Jain, Venkataramani, Srinivasan, Choi,
  Gopalakrishnan, and Chang]{biscaledDAC}
Jain, S., Venkataramani, S., Srinivasan, V., Choi, J., Gopalakrishnan, K., and
  Chang, L.
\newblock {BiScaled-DNN}: Quantizing long-tailed datastructures with two scale
  factors for deep neural networks.
\newblock In \emph{2019 56th ACM/IEEE Design Automation Conference (DAC)}, pp.\
   1--6. IEEE, 2019.

\bibitem[K{\"o}ster et~al.(2017)K{\"o}ster, Webb, Wang, Nassar, Bansal,
  Constable, Elibol, Hall, Hornof, Khosrowshahi, et~al.]{flexpoint}
K{\"o}ster, U., Webb, T., Wang, X., Nassar, M., Bansal, A.~K., Constable, W.,
  Elibol, O., Hall, S., Hornof, L., Khosrowshahi, A., et~al.
\newblock Flexpoint: An adaptive numerical format for efficient training of
  deep neural networks.
\newblock In \emph{Advances in Neural Information Processing Systems}, pp.\
  1740--1750, 2017.

\bibitem[LeCun et~al.(2015)LeCun, Bengio, and Hinton]{lecunNature}
LeCun, Y., Bengio, Y., and Hinton, G.
\newblock Deep learning.
\newblock \emph{nature}, 521\penalty0 (7553):\penalty0 436--444, 2015.

\bibitem[Lee et~al.(2017)Lee, Miyashita, Chai, Murmann, and
  Wong]{lee2017lognet}
Lee, E.~H., Miyashita, D., Chai, E., Murmann, B., and Wong, S.~S.
\newblock Lognet: Energy-efficient neural networks using logarithmic
  computation.
\newblock In \emph{2017 IEEE International Conference on Acoustics, Speech and
  Signal Processing (ICASSP)}, pp.\  5900--5904. IEEE, 2017.

\bibitem[Lin et~al.(2017)]{predictivenet}
Lin, Y. et~al.
\newblock {PredictiveNet}: an energy-efficient convolutional neural network via
  zero prediction.
\newblock In \emph{Circuits and Systems (ISCAS), 2017 IEEE International
  Symposium on}. IEEE, 2017.

\bibitem[Liu et~al.(2018)Liu, Wu, Luo, Yang, Liu, and Cheng]{birealnet}
Liu, Z., Wu, B., Luo, W., Yang, X., Liu, W., and Cheng, K.-T.
\newblock Bi-real net: Enhancing the performance of 1-bit cnns with improved
  representational capability and advanced training algorithm.
\newblock In \emph{Proceedings of the European conference on computer vision
  (ECCV)}, pp.\  722--737, 2018.

\bibitem[{Lloyd}(1982)]{lloyd_max}
{Lloyd}, S.
\newblock Least squares quantization in {PCM}.
\newblock \emph{IEEE Transactions on Information Theory}, 28\penalty0
  (2):\penalty0 129--137, 1982.

\bibitem[Nagel et~al.(2019)Nagel, Baalen, Blankevoort, and Welling]{nagelDFQ}
Nagel, M., Baalen, M.~v., Blankevoort, T., and Welling, M.
\newblock Data-free quantization through weight equalization and bias
  correction.
\newblock In \emph{Proceedings of the IEEE International Conference on Computer
  Vision}, pp.\  1325--1334, 2019.

\bibitem[Park \& Yoo(2020)Park and Yoo]{profitECCV2020}
Park, E. and Yoo, S.
\newblock Profit: A novel training method for sub-4-bit mobilenet models.
\newblock In \emph{European Conference on Computer Vision}, pp.\  430--446.
  Springer, 2020.

\bibitem[Paszke et~al.(2017)Paszke, Gross, Chintala, Chanan, Yang, DeVito, Lin,
  Desmaison, Antiga, and Lerer]{pytorch}
Paszke, A., Gross, S., Chintala, S., Chanan, G., Yang, E., DeVito, Z., Lin, Z.,
  Desmaison, A., Antiga, L., and Lerer, A.
\newblock Automatic differentiation in {PyTorch}.
\newblock In \emph{NeurIPS Workshop on Automatic Differentiation}, 2017.

\bibitem[Rajpurkar et~al.(2016)Rajpurkar, Zhang, Lopyrev, and Liang]{squad}
Rajpurkar, P., Zhang, J., Lopyrev, K., and Liang, P.
\newblock Squad: 100,000+ questions for machine comprehension of text.
\newblock In \emph{Proceedings of the 2016 Conference on Empirical Methods in
  Natural Language Processing}, pp.\  2383--2392, 2016.

\bibitem[Sakr \& Shanbhag(2019)Sakr and Shanbhag]{sakrICLRfx}
Sakr, C. and Shanbhag, N.~R.
\newblock Per-tensor fixed-point quantization of the back-propagation
  algorithm.
\newblock In \emph{7th International Conference on Learning Representations,
  ICLR 2019}, 2019.

\bibitem[Sakr \& Shanbhag(2021)Sakr and Shanbhag]{sakrTSP}
Sakr, C. and Shanbhag, N.~R.
\newblock Signal processing methods to enhance the energy efficiency of
  in-memory computing architectures.
\newblock \emph{IEEE Transactions on Signal Processing}, 69:\penalty0
  6462--6472, 2021.

\bibitem[Sakr et~al.(2017)]{sakrICML}
Sakr, C. et~al.
\newblock Analytical guarantees on numerical precision of deep neural networks.
\newblock In \emph{Proceedings of the 34th International Conference on Machine
  Learning}, pp.\  3007--3016, 2017.

\bibitem[Sakr et~al.(2019)]{sakrICLRfl}
Sakr, C. et~al.
\newblock Accumulation bit-width scaling for ultra-low precision training of
  deep networks.
\newblock In \emph{7th International Conference on Learning Representations,
  ICLR 2019}, 2019.

\bibitem[Sandler et~al.(2018)Sandler, Howard, Zhu, Zhmoginov, and
  Chen]{mobilenetv2}
Sandler, M., Howard, A., Zhu, M., Zhmoginov, A., and Chen, L.-C.
\newblock Mobilenetv2: Inverted residuals and linear bottlenecks.
\newblock In \emph{Proceedings of the IEEE conference on computer vision and
  pattern recognition}, pp.\  4510--4520, 2018.

\bibitem[Srivastava et~al.(2014)]{dropout}
Srivastava, N. et~al.
\newblock Dropout: A simple way to prevent neural networks from overfitting.
\newblock \emph{Journal of Machine Learning Research}, 15\penalty0
  (1):\penalty0 1929--1958, 2014.

\bibitem[Sun et~al.(2019)Sun, Choi, Chen, Wang, Venkataramani, Srinivasan, Cui,
  Zhang, and Gopalakrishnan]{IBMNIPS2019}
Sun, X., Choi, J., Chen, C.-Y., Wang, N., Venkataramani, S., Srinivasan, V.,
  Cui, X., Zhang, W., and Gopalakrishnan, K.
\newblock Hybrid 8-bit floating point {(HFP8)} training and inference for deep
  neural networks.
\newblock In \emph{NeurIPS}, 2019.

\bibitem[Taigman et~al.(2014)]{deepface}
Taigman, Y. et~al.
\newblock {Deepface: Closing the gap to human-level performance in face
  verification}.
\newblock In \emph{Proceedings of the IEEE Conference on Computer Vision and
  Pattern Recognition}, pp.\  1701--1708, 2014.

\bibitem[Tambe et~al.(2019)Tambe, Yang, Wan, Deng, Reddi, Rush, Brooks, and
  Wei]{tambeAdaptivfloat}
Tambe, T., Yang, E.-Y., Wan, Z., Deng, Y., Reddi, V.~J., Rush, A., Brooks, D.,
  and Wei, G.-Y.
\newblock Adaptivfloat: A floating-point based data type for resilient deep
  learning inference.
\newblock \emph{arXiv preprint arXiv:1909.13271}, 2019.

\bibitem[Wang et~al.(2018)Wang, Choi, Brand, Chen, and
  Gopalakrishnan]{ibmNips18}
Wang, N., Choi, J., Brand, D., Chen, C.-Y., and Gopalakrishnan, K.
\newblock Training deep neural networks with 8-bit floating point numbers.
\newblock In \emph{Advances in Neural Information Processing Systems}, 2018.

\bibitem[Widrow \& Koll{\'a}r(2008)Widrow and Koll{\'a}r]{widrowBook}
Widrow, B. and Koll{\'a}r, I.
\newblock Quantization noise.
\newblock \emph{Cambridge University Press}, 2:\penalty0 5, 2008.

\bibitem[{Wikimedia Foundation}(2021)]{wikipedia}
{Wikimedia Foundation}.
\newblock Wikimedia downloads, 2021.
\newblock URL \url{https://dumps.wikimedia.org}.

\bibitem[Wu et~al.(2020)Wu, Judd, Zhang, Isaev, and
  Micikevicius]{wu2020integer}
Wu, H., Judd, P., Zhang, X., Isaev, M., and Micikevicius, P.
\newblock Integer quantization for deep learning inference: Principles and
  empirical evaluation.
\newblock \emph{arXiv preprint arXiv:2004.09602}, 2020.

\bibitem[Zhang et~al.(2018)Zhang, Yang, Ye, and Hua]{lqnet}
Zhang, D., Yang, J., Ye, D., and Hua, G.
\newblock {LQ-Nets}: Learned quantization for highly accurate and compact deep
  neural networks.
\newblock In \emph{Proceedings of the European Conference on Computer Vision
  (ECCV)}, pp.\  365--382, 2018.

\bibitem[Zhao et~al.(2021)Zhao, Dai, Venkatesan, Liu, Khailany, Dally, and
  Anandkumar]{nvidiaLNS}
Zhao, J., Dai, S., Venkatesan, R., Liu, M.-Y., Khailany, B., Dally, B., and
  Anandkumar, A.
\newblock Low-precision training in logarithmic number system using
  multiplicative weight update.
\newblock \emph{arXiv preprint arXiv:2106.13914}, 2021.

\bibitem[Zhou et~al.(2016)Zhou, Wu, Ni, Zhou, Wen, and Zou]{dorefa}
Zhou, S., Wu, Y., Ni, Z., Zhou, X., Wen, H., and Zou, Y.
\newblock {DoReFa-Net}: Training low bitwidth convolutional neural networks
  with low bitwidth gradients.
\newblock \emph{arXiv preprint arXiv:1606.06160}, 2016.

\bibitem[Zhu et~al.(2015)Zhu, Kiros, Zemel, Salakhutdinov, Urtasun, Torralba,
  and Fidler]{bookcorpus}
Zhu, Y., Kiros, R., Zemel, R., Salakhutdinov, R., Urtasun, R., Torralba, A.,
  and Fidler, S.
\newblock Aligning books and movies: Towards story-like visual explanations by
  watching movies and reading books.
\newblock In \emph{Proceedings of the IEEE international conference on computer
  vision}, pp.\  19--27, 2015.

\end{thebibliography}
\bibliographystyle{icml2022}
\newpage
\begin{center}\framebox[1.1\width][c]{\Large Supplementary Material}\end{center}
\appendix
The main section of our paper contains all necessary content to understand our work. This supplementary section is provided to the interested reader as an avenue to go further. Here, we include an extension of all mathematical results to the unsigned quantization case, proofs of various theorems and propositions, implementation details needed to reproduce our experimental results, timing comparisons between OCTAV and the MSE sweep, and an investigation on the occurrence of non-convex clipped quantization MSE.
\section{Results for Unsigned Quantization}
\label{appendix:unsigned}
All mathematical results presented in the main paper made the assumption that any data being quantized was signed, i.e., lying in $[-s_{\max},s_{\max}]$. A study of the unsigned case is also relevant, as unsigned activations (such as ReLU) are commonly employed. Hereafter, we list all required modifications to our mathematical results for the unsigned case, i.e., quantization over $[0,s_{\max}]$.

We start with the max-scaled quantization operation, which, rather than \eqref{eqn:max_scaled_quantization}, is given by:
\begin{align*}
    \mathbbm{Q}(x) ={s_{\max}} \cdot 2^{-B}\cdot \text{round}\left({x\cdot2^{B}}/{s_{\max}} \right)
\end{align*}
and its MSE, derived from an additive model of quantization noise is given by $J=s_{\max}^2\frac{4^{-B}}{12}$. Effectively, when data is unsigned, there is no sign bit, and instead an additional least-significant-bit. Thus, for an identical value of $s_{\max}$, an unsigned quantizer has a twice smaller quantization step and consequently a $4\times$ smaller MSE than its signed counterpart.

Similarly, the unsigned clipped quantization operation, using a clipping scalar $s<s_{\max}$ is given by:
\begin{align*}
    \mathbbm{Q}(x) &=  \text{Rclip}\left({s} \cdot 2^{-B}\cdot\text{round}\left({x\cdot2^{B}}/{s}\right),s\right) \nonumber \\
    &=\begin{cases}{s} \cdot 2^{-B}\cdot
    \text{round}\left({x\cdot2^{B}}/{s}\right) \quad \text{if } x\in[0,s]\\
    s \quad \text{if } x>s
    \end{cases}
\end{align*}
where $\text{Rclip}()$ is the right clipping operator used in lieu of the usual two-sided clipping operator. Further, the unsigned clipped quantization MSE can be re-derived, and rather than \eqref{eqn:clipped_quantization_mse}, it is given by:
\begin{align*}
    J(s) = \frac{4^{-B}}{12}s^2\int_0^sf_{X}(x)dx + \int_s^{\infty}(s-x)^2f_{X}(x)dx
\end{align*}
where we note the factor of 12 rather than 3 in the first term and the use of $f_X()$ rather than $f_{|X|}()$; since the data is unsigned, there is no need to consider absolute values.

With the above expression for the unsigned clipped quantization MSE, Theorem \ref{thm:octav} can be re-derived in the same manner as Appendix \ref{appendix:octav_proof} below, and it can be readily shown that the recursion in \eqref{eqn:octav_theorem} is modified to:
\begin{align*}
    s_{n+1} = \frac{\mathbbm{E}\left[X\cdot \mathbbm{1}_{\{X>s_n\}} \right]}{\frac{4^{-B}}{12}\mathbbm{E}\left[ \mathbbm{1}_{\{X\leq s_n\}}\right] +\mathbbm{E}\left[ \mathbbm{1}_{\{X>s_n\}}\right]}
\end{align*}

and similarly, the recursion \eqref{eqn:octav_corr} in Corollary \ref{corr:octav} is modified to:
\begin{align*}
    &s_{n+1} =
    \frac{\sum_{x\in\vec{t}}\left[x\cdot \mathbbm{1}_{\{x>s_n\}} \right]}{\frac{4^{-B}}{12}\sum_{x\in\vec{t}}\left[ \mathbbm{1}_{\{0<x\leq s_n\}}\right] +\sum_{x\in\vec{t}}\left[ \mathbbm{1}_{\{x>s_n\}}\right]}
\end{align*}
where in both cases, we note the use of 12 rather than 3 in the first denominator term, and the lack of need for absolute values.

Finally, we append to Proposition \ref{prop:mad_clipping} the result related to the right clipping operator, which is given by:
\begin{align*}
    \text{Rclip}(x,s) = \left(\mathbbm{1}_{\{x\leq s\}} + \frac{s}{x}\mathbbm{1}_{\{x>s\}}\right)\cdot x
\end{align*}
and, consequently, for unsigned activations, rather than \eqref{eqn:mad_derivative}, the MAD estimator is given by:
\begin{align*}
    \frac{\partial^{\text{(MAD)}}\mathbbm{Q}(x)}{x} =  \mathbbm{1}_{\{x \leq s\}} + \frac{s}{x}\mathbbm{1}_{\{x>s\}}
\end{align*}
with the only modification made being the removal of absolute values.

\section{Proof of Theorem \ref{thm:octav}}
\label{appendix:octav_proof}
Our strategy is to use the Newton-Raphson algorithm. The algorithm consists of selecting a random guess $s_1$ and iteratively computing $\{s_n\}_{n>1}$ until convergence using the following recursion:
\begin{align}
\label{eqn:newton_raphson}
s_{n+1} = s_n - \frac{J'(s_n)}{J''(s_n)}
\end{align}
This recursion requires expressions for the first and second derivatives of the clipped quantization MSE $J(s)$. Obtaining closed form expressions for these derivatives is possible, but complicated. Thus, we first re-write the clipped quantization MSE in \eqref{eqn:clipped_quantization_mse} as follows:
\begin{align}
\label{eqn:mse_indicator}
    J(s) = \frac{4^{-B}}3s^2\mathbbm{E}\left[\mathbbm{1}_{\{|X|\leq s\}} \right] + \mathbbm{E}\left[(s-|X|)^2\mathbbm{1}_{\{|X|>s\}} \right]
\end{align}
No approximation was made. Indeed, \eqref{eqn:clipped_quantization_mse} and \eqref{eqn:mse_indicator} are identical by virtue of the relationship between integrals of distributions and expected values of indicator functions. More generally, we used:
\begin{align*}
    \int_{\mathcal{R}} g(x)f_X(x)dx = \mathbbm{E}\left[g(X)\mathbbm{1}_{\{X\in\mathcal{R}\}} \right]
\end{align*}
for a given interval $\mathcal{R}$ and function $g()$. We now note that the expression for $J(s)$ in \eqref{eqn:mse_indicator} is easier to differentiate. Nonetheless, it requires differentiating the indicator function ($\mathbbm{1}_{\{|X|\leq s\}}$ and $\mathbbm{1}_{\{|X|>s\}}$). This function is technically non-differentiable; however, it is customary in deep learning to treat it as a piecewise linear function. For instance, such is the way the ReLU function is handled for training neural networks using the back-propagation algorithm. In our case, we do make the piecewise linear approximation to obtain: $\frac{\partial \mathbbm{1}_{\{|X|\leq s\}}}{\partial s} = \frac{\partial \mathbbm{1}_{\{|X|>s\}}}{\partial s} = 0$. Because the derivative is taken inside an expectation operation in \eqref{eqn:mse_indicator}, this approximation is generally valid. The validity of the approximation would fail if one of the candidates $\{s_n\}_{n>1}$ coincides with a point of positive mass. Thus, to be certain that the algorithm converges correctly, one requirement on the distribution $f_X()$ would be to have no points of positive mass in the vicinity of $s_1$. This is by no means a strong condition, and we expect all data of interest to satisfy this condition anyway. Next, we simply compute the first and second order derivatives of $J(s)$ in \eqref{eqn:mse_indicator} to obtain:
\begin{align*}
    &J'(s) = \frac{2\cdot 4^{-B}}{3}s\mathbbm{E}\left[\mathbbm{1}_{\{|X|\leq s\}} \right] + 2\mathbbm{E}\left[(s-|X|)\mathbbm{1}_{\{|X|>s\}} \right] \\
    &J''(s) = \frac{2\cdot 4^{-B}}{3}\mathbbm{E}\left[\mathbbm{1}_{\{|X|\leq s\}} \right] + 2\mathbbm{E}\left[\mathbbm{1}_{\{|X|>s\}} \right]
\end{align*}
which we insert into \eqref{eqn:newton_raphson} to obtain \eqref{eqn:octav_theorem}. Observe that the second derivate $J''(s)$ is a positive quantity, proving the convexity of $J(s)$ and justifying the use of the Newton-Raphson method. It might be useful to notice the following:
\begin{align*}
    J'(s) = s\cdot J''(s) - 2\mathbbm{E}\left[|X|\cdot \mathbbm{1}_{\{|X|>s_n\}} \right]
\end{align*}

\section{Proof of Proposition \ref{prop:grad_explosion}}
\label{appendix:grad_explosion_proof}
Note that $\text{Var}\left(\Delta X_l \right) = k_l \mathbbm{E}\left[\left|\frac{\partial \mathbbm{Q}(X_l)}{\partial X_l} \right|^2 \right]$ where $k_l$ accumulates the back-propagated gradient variance until the quantization step at layer $l$. It follows that:
\begin{align*}
    \frac{\text{Var}\left(\Delta^{\text{(STE)}}X_l \right)}{\text{Var}\left(\Delta X_l \right)} = \prod_{i=l}^L\frac{\mathbbm{E}\left[\left|\frac{\partial^{\text{(STE)}} \mathbbm{Q}(X_i)}{\partial X_i} \right|^2 \right]}{\mathbbm{E}\left[\left|\frac{\partial \mathbbm{Q}(X_i)}{\partial X_i} \right|^2 \right]}
\end{align*}
Indeed, at every layer, the gradient estimator replaces the true gradient. Further, note that $\mathbbm{E}\left[\left|\frac{\partial^{\text{(STE)}} \mathbbm{Q}(X_i)}{\partial X_i} \right|^2 \right]=1$ by definition of the STE (see Section \ref{sec:gradients}). By virtue of the law of total expectation, we have:
\begin{align*}
    &\mathbbm{E}\left[\left|\frac{\partial \mathbbm{Q}(X_i)}{\partial X_i} \right|^2 \right] \\ &= \mathbbm{E}\left[\left|\frac{\partial \mathbbm{Q}(X_i)}{\partial X_i} \right|^2 \Bigg| |X_i|\leq s_i \right]\Pr\left(|X_i|\leq s_i\right) \\ & + \mathbbm{E}\left[\left|\frac{\partial \mathbbm{Q}(X_i)}{\partial X_i} \right|^2 \Bigg| |X_i|>s_i \right]\Pr\left(|X_i|>s_i\right)
\end{align*}
In the discretization region, and for sufficiently small quantization step, which is in line with our paper's assumptions, we have:
$\mathbbm{E}\left[\left|\frac{\partial \mathbbm{Q}(X_i)}{\partial X_i} \right|^2 \Bigg| |X_i|\leq s_i \right]=1$. In contrast, in the clipping regions, we have $\mathbbm{E}\left[\left|\frac{\partial \mathbbm{Q}(X_i)}{\partial X_i} \right|^2 \Bigg| |X_i|>s_i \right]=0$. We obtain:
\begin{align*}
    \mathbbm{E}\left[\left|\frac{\partial \mathbbm{Q}(X_i)}{\partial X_i} \right|^2 \right] &= \Pr\left(|X_i|\leq s_i\right) \\&< 1 = \mathbbm{E}\left[\left|\frac{\partial^{\text{(STE)}} \mathbbm{Q}(X_i)}{\partial X_i} \right|^2 \right]
\end{align*}
The strict inequality occurs when using any clipped quantizer, but not the max-scaled quantizer. Thus, we obtain:
\begin{align*}
    \frac{\mathbbm{E}\left[\left|\frac{\partial^{\text{(STE)}} \mathbbm{Q}(X_i)}{\partial X_i} \right|^2 \right]}{\mathbbm{E}\left[\left|\frac{\partial \mathbbm{Q}(X_i)}{\partial X_i} \right|^2 \right]} = 1+\delta_i
\end{align*}
where $\delta_i>0$. We assign $\delta = \min_{i=1\ldots L}\delta_i$ to obtain the desired result in \eqref{eqn:grad_explosion} of Proposition \ref{prop:grad_explosion}.

\section{Proof of Proposition \ref{prop:pwl_stoppage}}
\label{appendix:pwl_stoppage_proof}
In Proposition \ref{prop:pwl_stoppage}, static clipped quantization is assumed, thus, the quantization scalar satisfies $s<s_{\max}$. It follows that, there exists at least one weight element $w_{\circ} \in \vec{w}$ such that its initial value at iteration 1 satisfies $\left|w_{\circ}^{(1)}\right|>s$. Thus, its first PWL gradient estimate at iteration 1 is given by: $\Delta^{\text{(PWL)}}w_{\circ}^{(1)} = 0$. Consequently, $w_{\circ}^{(2)}=w_{\circ}^{(1)}$ and $\Delta^{\text{(PWL)}}w_{\circ}^{(2)} = 0$. More generally, we may use induction to obtain that $\forall i, \Delta^{\text{(PWL)}}w_{\circ}^{(i)} = 0$ and the weight value is 'stuck at initialization', i.e., $w_{\circ}^{(i)}=w_{\circ}^{(1)}$. This means that $w_{\circ}$ is never learned and therefore $N_{\vec{w}} > \tilde{N}^{(i)}_{\vec{w}}$, which is the first inequality in Proposition \ref{prop:pwl_stoppage}. This first part of the result also applies to dynamic quantization. Furthermore, if at iteration $i$ there is at least one $w_* \in \vec{w}$ such that $\left|w_*^{(i)}\right|<s$ and  $\Delta^{\text{(PWL)}}w_{*}^{(i)}$ causes the next iteration weight magnitude to be $\left|w_*^{(i+1)}\right|>s$, then $\tilde{N}^{(i)}_{\vec{w}} > \tilde{N}^{(i+1)}_{\vec{w}}$. If no such weight exists, then $\tilde{N}^{(i)}_{\vec{w}} = \tilde{N}^{(i+1)}_{\vec{w}}$. This completes the proof of the second inequality in Proposition \ref{prop:pwl_stoppage}. At any iteration $i$, at best the number of learnable parameters remains constant, and at worst, it decreases.

\section{Experimental Implementations Details}
\label{appendix:implementations_details}
In this section, we discuss all details behind our implementations in Section \ref{sec:experiments}. These include training recipes, tensor quantization specifics, and static quantization calibration details. All experiments were implemented using DGX-1 Volta machines. 
\subsection{Baseline Training Recipes}
The following recipes were used for obtaining ImageNet \textbf{training-from-scratch} results included in Section \ref{sec:training_from_scratch}. All models were trained using momentum SGD using 8 V100 GPUs. The following parameters are the defaults suggested in the `Deep Learning Examples' repository \footnote{Code retrieved from \url{https://github.com/NVIDIA/DeepLearningExamples/tree/master/PyTorch/Classification/ConvNets}.}.

\textbf{ResNet} training used the following parameters: an initial learning of 0.1, a per-GPU batch-size of 64, a momentum factor of 0.9, a weight decay factor of 1e-4, and a learning rate decay factor of 0.1 every 30 epochs. ResNet-50 and ResNet-18 were trained for 150 epochs, while ResNet-101 was trained for 80 epochs.

\textbf{MobileNet} training used the following parameters: an initial learning rate of 0.2, a per-GPU batch-size of 128, a momentum factor of 0.9, a learning weight decay factor of 0.1 every 100 epochs, and a total number of training epochs of 300. Weight decay factors of 4e-5 and 1e-5 were used for MobileNet-V2 and MobileNet-V3, respectively.

Shortened versions of the above recipes were used for \textbf{retraining} discussed in Section \ref{sec:re_training}. Furthermore, models were retrained using 4 V100 GPUs. More specifically, ResNets and MobileNets were retrained using an initial learning rate of 1e-4 and 1e-2, respectively. ResNet-50 and ResNet-101 were retrained for 15 epochs with a learning rate decay of 0.1 every 5 epochs. All other models were retrained for 30 epochs with a learning rate decay of 0.1 every 10 epochs. 

For \textbf{fine-tuning} of BERT models on Squad, we used 1 V100 GPU and referred to the default parameters suggested in the `Deep Learning Examples' repository \footnote{Code retrieved from \url{https://github.com/NVIDIA/DeepLearningExamples/tree/master/PyTorch/LanguageModeling/BERT}.}. Specifically, we fine-tune for 2 epochs and use: a batch of 4, a learning rate of 3e-5, and a maximum sequence length of 384.

\subsection{Tensor Quantization Specifics}
It is important to explicitly describe various options for quantization granularity. The following are commonly employed methods:
\begin{itemize}[leftmargin=*,itemsep=0.1\baselineskip,topsep=0pt]
    \item \textbf{Tensor scaling}, where one scalar $s$ is used to quantize all elements in a tensor $\vec{t}$.
    \item \textbf{Per-output-channel (POC) scaling}, which applies to weights in convolutional layers. A convolutional weight tensor $\vec{w}$ has a dimensionality of $(K,C,R,S)$ where $K$ and $C$ are the output and input number of channels, respectively, and $R$ and $S$ are the kernel height and width, respectively. A vector $\vec{s}$ of $K$ scalars is used to quantize $C\times R\times S$ elements at a time. Specifically, the tensor $\vec{w}$ is flattened into a matrix of dimensionality $(K,C\times R\times S)$ and every row vector is quantized using the scalar in $\vec{s}$ of appropriate index.
    \item \textbf{Per-output-feature (POF) scaling}, which applies to weights in linear (fully connected) layers. A weight matrix $\vec{w}$ has dimensionality $(O,I)$ where $O$ and $I$ are the output and input number of features, respectively. A vector $\vec{s}$ of $O$ scalars is used to quantize the row vectors of $\vec{w}$ in a similar fashion as the POS scaling case.
\end{itemize}
The advantage of POS and POF scaling is the finer granularity of quantization, which can boost accuracy. The above three methods are well-adapted to GPU and deep learning accelerator datapaths \cite{wu2020integer}. In addition, the overhead of quantization metadata is small enough that in our work we assume all scalars to be in full precision \cite{daiVSQuant}.

In our experiments, we use POC scaling for convolutional weight quantization in all ResNets and MobileNets, and POF scaling for the weight layers in BERT models. Tensor scaling is used for all other tensors, which include: activations and fully connected layers in ResNets and MobileNets, and activations in BERT models. For the last, we define activations as being inputs to the various linear layers in the transformer blocks, as well as inputs to the Batch Matrix Multiply (BMM) operations.

Finally, we note that first (convolutional) and last (fully connected) layers in ResNets and MobileNets are always quantized to 8-bit, as is customary for such networks \cite{pact}. However, projection layers (also known shortcut connections) are quantized to low precision (4-bit, 6-bit, and 8-bit, depending on the experiment).

\subsection{Static Quantization Calibration Details}
In this section, we describe the methodology to calibrate clipping scalars for retraining and fine-tuning experiments in Section \ref{sec:re_training} and \ref{sec:fine_tuning}, respectively.

Weight calibration is done by simply using the pretrained weights at various layers and computing various clipping scalar candidates. Specifically, the candidates used are the OCTAV and MSE sweep predicted scalars, as well as the $99.9^{\text{th}}$, $99.99^{\text{th}}$, and $99.999^{\text{th}}$ percentiles.

For activations, we sample 5 random input batches from the training set. For each, we compute the various candidates above at every layer. Thus, for each strategy, we obtain a collection of 5 candidates for every activation tensor. The average of these candidates is used as a calibrated clipping scalar. For BERT-Base, we sample and average over 10 random inputs from the training set.

We define the MSE Sweep as being a 100-point sweep. Specifically, for every tensor, we evaluate the sweep over $\{\frac{k}{100}\cdot s_{\max}\}_{k=1}^{100}$. In the case of POC and POF scaling, we run the sweep over $\{\frac{k}{100}\cdot \vec{s}_{\max}\}_{k=1}^{100}$ where $\vec{s}_{\max}$ is the vector of row-wise maxima in the tensors. This is an approximation leading to a weaker solution than OCTAV, which, thanks to broadcasting, can perform row-wise optimization with no overhead. Nonetheless, we find both to yield similar results, as highlighted in our results in Section \ref{sec:re_training}.

\section{OCTAV vs. Brute Force Sweep Speed Comparison}
\label{appendix:timing_comparison}
We compare times taken to calibrate BERT-Base tensors when using OCTAV vs. the MSE sweep. The calibration was done on an Intel Xeon CPU, using the NumPy package, and no code optimization was performed for either. We choose to measure time on a CPU to provide the fairest and most accurate comparison possible. Indeed, on a GPU, significant noise in time measurement can occur due to various communication of data between GPU and host that may be required. As OCTAV only requires tensor operations, its GPU speedup over MSE sweeps is expected to be even greater than what we report here.

In BERT-Base, there are 74 weight layers and thus 74 weight tensors to be processed. For each of these layers, we also process an input activation tensor. As we use 10 random input batches, we obtain 740 activation tensors to process. Furthermore, there are 24 BMM operations, each of which has two operands that are added to the list of activations. Thus, we get another 48 activation tensors as BMM operands per input batch, for a total of 480 additional activation tensors. Thus, we process 1220 activation tensors overall.

As we perform calibration for $B=4,5,6,7,8$, we report average times per-precision. For each tensor, we start a timer as soon as we invoke our calibration (OCTAV vs. MSE sweep) routine and stop it as soon as execution terminates. 

\begin{table}[!t]
    \centering
    \caption{Calibration times comparison of OCTAV vs. MSE Sweep for BERT-Base}
    \label{table:calibration_times}
    \begin{tabular}{|c|c|c|c|}
    \cline{3-4}
    \multicolumn{2}{c|}{~} & Weight & Activation\\
    \hline
    \multirow{2}{*}{\specialcell{Per-Tensor \\ Avg. (seconds)}}  & OCTAV & 0.191 & 0.497 \\
    \cline{2-4} 
    ~& MSE Sweep & 1.943 & 3.141\\
    \hline
    \multirow{2}{*}{\specialcell{Total Calib. \\  (h:mm:ss)}} & OCTAV & 0:00:14 & 0:10:06\\
    \cline{2-4}
    & MSE Sweep & 0:02:24 & 1:03:52\\
    \hline
    \multicolumn{2}{|c|}{Speed-up} & 10.2$\times$& 6.3$\times$ \\
    \hline
    \end{tabular}
\end{table}

Our measurements are listed in Table \ref{table:calibration_times}. We find that OCTAV processes one weight tensor in 0.191 sec., on average completing all weight calibration in only 14 sec. In contrast, the MSE sweep requires over 2 mins. and is 10.2 $\times$ slower. For activations, OCTAV processes one tensor in 0.497 sec. on average, and completes all activation calibration in just over 10 mins. This corresponds to the calibration of 1220 tensors on a CPU. In contrast, the MSE sweep is 6.3 $\times$ slower and requires over 1 hour to terminate. Interestingly, activation calibration speed-up is less pronounced than that of weights. We speculate that this is due to the activation tensors being much larger so that much of the execution time goes into data movement from memory, which amortizes the speed-up.

\section{When is MSE not Convex?}
\label{appendix:mse_convexity}
Comparing static quantization using OCTAV vs. MSE sweep results for BERT models in Section \ref{sec:fine_tuning} reveals unexpected results. As shown in Table \ref{table:bert_combined}, OCTAV calibration at a low precision yields a significantly higher accuracy than the MSE sweep. As both methods are supposed to return similar solutions, we investigated the reason of this discrepancy.

\begin{figure}[!t]
\begin{center}
    \begin{subfigure}[t]{0.45\textwidth}
    \centering
        \includegraphics[width = 0.7\linewidth]{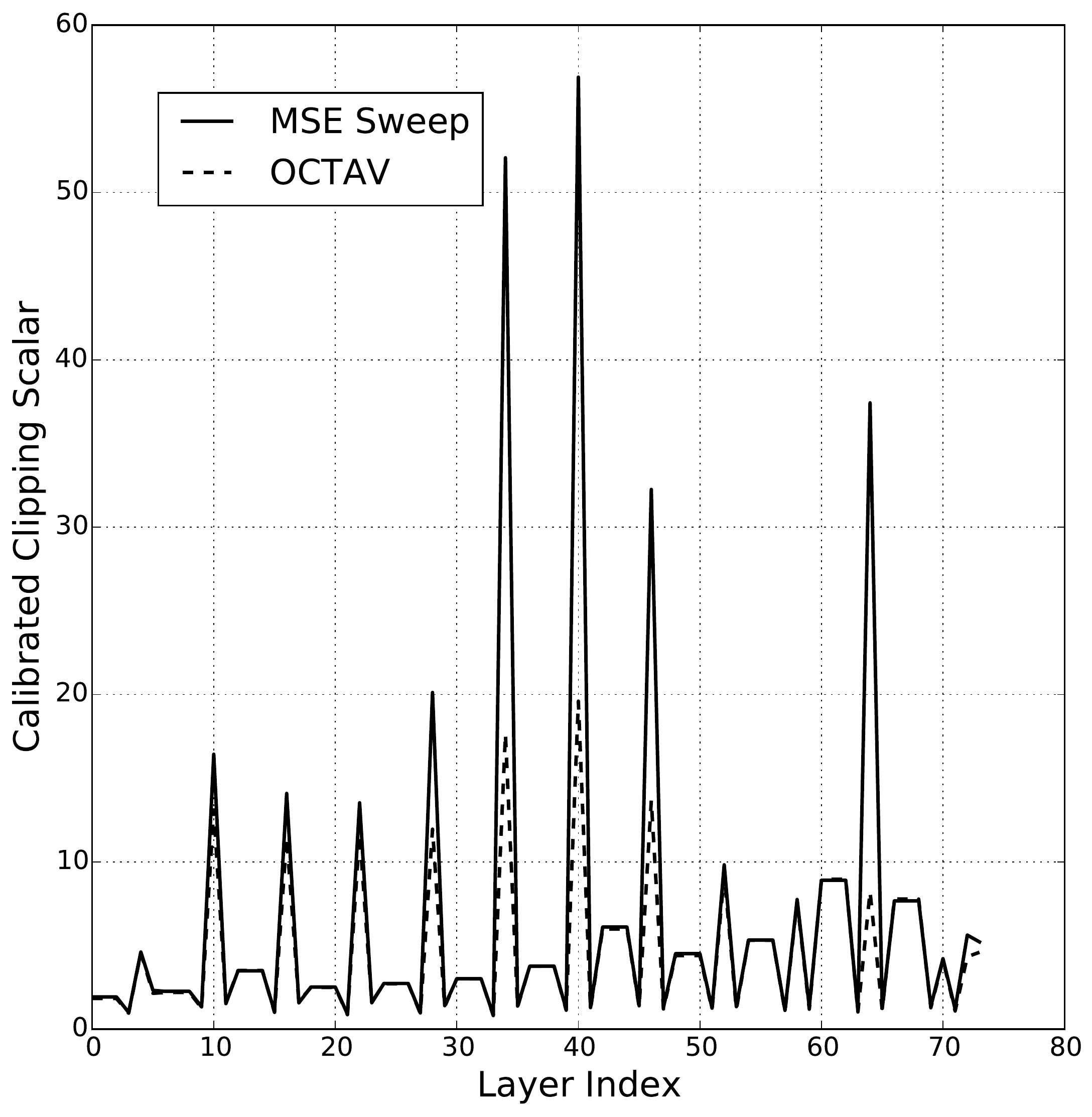}
    \caption{}
    \end{subfigure}\\
    \begin{subfigure}[t]{0.45\linewidth}
    \centering
        \includegraphics[width = 0.99\linewidth]{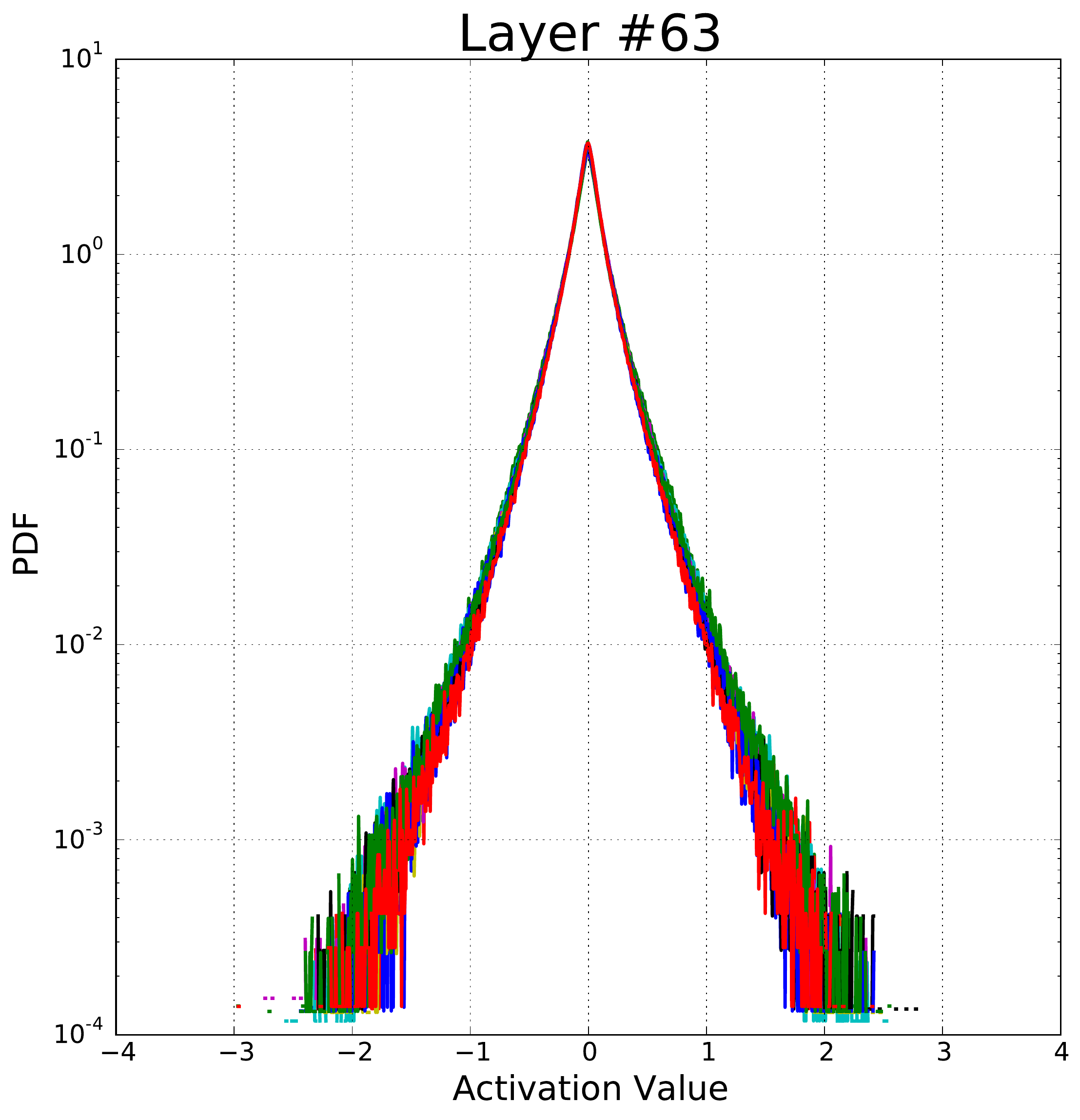}
    \caption{}
    \end{subfigure}~
    \begin{subfigure}[t]{0.45\linewidth}
    \centering
        \includegraphics[width = 0.99\linewidth]{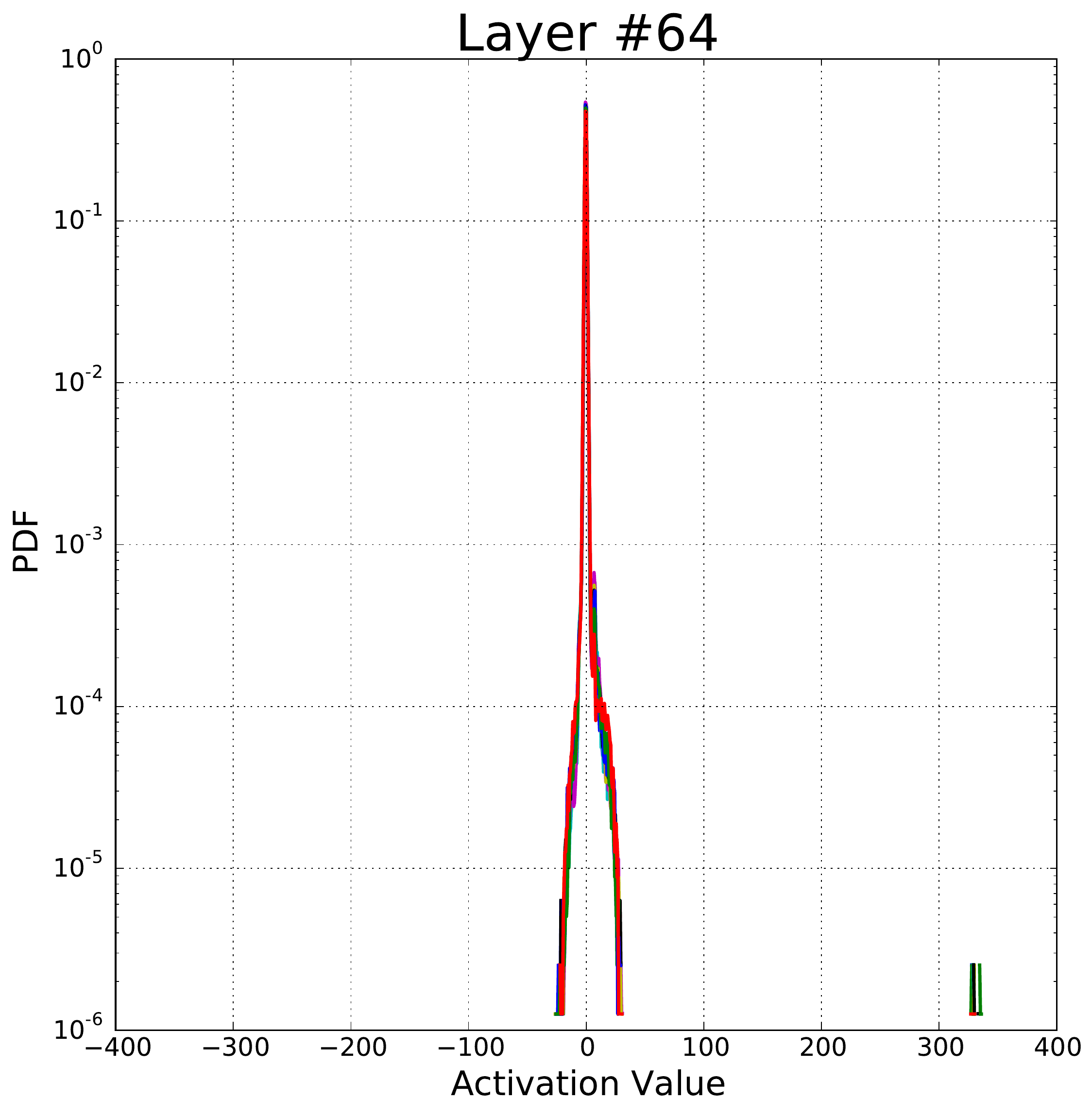}
    \caption{}
    \end{subfigure}\\
    \begin{subfigure}[t]{0.45\linewidth}
    \centering
        \includegraphics[trim=7.5cm 0 7.5cm 0, clip,width = 0.99\linewidth,page=8]{figures/paper_figures.pdf}
    \caption{}
    \end{subfigure}~
    \begin{subfigure}[t]{0.45\linewidth}
    \centering
        \includegraphics[trim=7.5cm 0 7.5cm 0, clip,width = 0.99\linewidth,page=9]{figures/paper_figures.pdf}
    \caption{}
    \end{subfigure}
\end{center}
\caption{\footnotesize Investigation of 4-bit clipped quantization MSE convexity in BERT-Base: (a) calibrated clipping scalars returned by OCTAV and MSE sweep for activation layers, (b) PDF of activation tensors at layer 63, (c) PDF of activation tensors at layer 64, (c) MSE vs. clipping scalar at layer 63, and (d) MSE vs. clipping scalar at layer 64. Different colors correspond to the 10 different input batches used for calibration. Activation tensors considered are input to the linear layers in BERT-Base, while BMM operands are not shown.}
\label{fig:bert_base_convexity}
\end{figure}

In Figure \ref{fig:bert_base_convexity}(a) we plot the calibrated clipping scalar for activations (linear layers input only, BMM operands excluded) as a function of layer index when using OCTAV and the MSE sweep. We observe that the solution is identical almost everywhere, except at a few layers. Specifically, at layers 28, 33, 40, and 64, the OCTAV-calibrated scalar is noticeably smaller than its MSE sweep counterpart. To understand this phenomenon, we carefully compared the data at layers 63 and 64. For the former, the two calibrated strategies converge to the same solution, but for the latter, OCTAV returns a clipping scalar of $\sim$8 while the MSE sweep returns $\sim$38.

In Figure \ref{fig:bert_base_convexity}(b,c), we plot the empirical probability distribution function (PDF) of the tensors used for calibration at layer 63 and 64, respectively. These tensors correspond to the 10 random input batches selected for calibration, each of which is represented by one color. A stark contrast in distribution is observed: layer 63 appears to be typical with most of the density concentrated around zero, while layer 64 data has large outliers around the value 350.

In Figure \ref{fig:bert_base_convexity}(d,e), we plot the 4-bit clipped quantization MSE as a function of clipping scalars for layers 63 and 64, respectively. We also pinpoint the predicted clipping scalar by OCTAV for each of the 10 tensors. As expected, for layer 63, the MSE is a convex function and OCTAV accurately converges to its global optimum. In contrast, layer 64 exhibits a more complex behavior. The MSE does include a local minimum close to zero, which corresponds to the trade-off between discretization and clipping noise, and to which OCTAV converges to. Beyond that point, the MSE increases as the clipping scalar increases, and this is due to an increase in the discretization step. At some point, non-outlier data is bound to become smaller than the magnitude of the least significant bit and thus be quantized to zero. At this point, the quantization noise  related to non-outlier data quantization is equal to its total variance and no longer increases when the clipping scalar increases. However, with larger clipping scalar, the quantization noise of outliers themselves decreases, which leads to the MSE decreasing again. The MSE therefore has a second minimum, which is close to the maximum scalar and only caters for outlier representation. Furthermore, for one of the 10 tensors processed, in this case, the one corresponding to the green line in Figure \ \ref{fig:bert_base_convexity}(d), the MSE at this second minimum is smaller than that of the first one. This second minimum is thus selected by the MSE sweep and skews the calibrated scalar towards the outliers.

In conclusion, when tensors have large outliers, the MSE sweep may return a calibrated scalar that caters only for outliers and zeroes out all small values. In contrast, OCTAV converges to the minimum closest to zero, which balances discretization and clipping noise of all data in the tensor. This is in fact the only minimum identified by $J(s)$ in \eqref{eqn:clipped_quantization_mse}, due to the additive noise model assumption. Thus, OCTAV is guaranteed to converge to this first minimum. Intuitively, the solution returned by OCTAV is desired, as it caters for \textit{all} the data in the tensor. This intuition concurs with our experimental results in Section \ref{sec:fine_tuning}. Indeed, static-OCTAV was found to yield a noticeably higher accuracy than the MSE sweep for both BERT models at low precision (4-to-6-bit).

\end{document}